\documentclass[11pt]{article}
\usepackage[english]{babel}
\usepackage[letterpaper,margin=1in]{geometry}
\usepackage[T1]{fontenc}
\usepackage{lmodern}
\usepackage{amsmath,amssymb,amsthm,mathtools,booktabs,array}
\usepackage{microtype}
\usepackage{tikz}
\usetikzlibrary{arrows.meta}
\usepackage{authblk}
\usepackage[colorlinks=true,allcolors=blue]{hyperref}
\hypersetup{pdftitle={Stochastic Autoregressive Learning},
	pdfauthor={Ilan Doron-Arad, Idan Mehalel, Elchanan Mossel},
	pdfsubject={Forward-KL statistical and computational separations}}
\newtheorem{theorem}{Theorem}[section]
\newtheorem{lemma}[theorem]{Lemma}
\newtheorem{corollary}[theorem]{Corollary}
\newtheorem{proposition}[theorem]{Proposition}
\theoremstyle{definition}

\theoremstyle{plain}
\newcommand{\cF}{\mathcal F}

\newcommand{\ct}{\mathsf{CoT}}
\newcommand{\etoe}{\mathsf{e2e}}
\newcommand{\iter}{M}

\newcommand{\KL}{\operatorname{KL}}
\newcommand{\Ber}{\operatorname{Bern}}
\newcommand{\base}{\mathrm{base}}
\newcommand{\tail}{\operatorname{tail}}

\numberwithin{equation}{section}
\AtBeginDocument{\setlength{\abovedisplayskip}{8pt plus 2pt minus 2pt}%
	\setlength{\belowdisplayskip}{8pt plus 2pt minus 2pt}%
	\setlength{\abovedisplayshortskip}{3pt plus 2pt}%
	\setlength{\belowdisplayshortskip}{5pt plus 2pt minus 2pt}}
\allowdisplaybreaks[1]
\newenvironment{informalresult}[2]{%
	\par\medskip\noindent\textbf{#1~\ref{#2} (informal).}\itshape
}{\par\medskip}
\title{Stochastic Autoregressive Learning}
\author[1]{Ilan Doron-Arad}
\author[2]{Idan Mehalel}
\author[1]{Elchanan Mossel}
\affil[1]{MIT}
\affil[2]{The Hebrew University}
\date{}
\begin{document}
	\maketitle
	\begin{abstract}
		Motivated by large language models (LLMs), which generate outputs by iteratively sampling from next-token distributions, we introduce a PAC-learning model for binary stochastic autoregressive learning. This generalizes the deterministic autoregressive learning framework of Joshi et al., COLT 2025. In our model, one fixed Bernoulli next-token generator is iterated for $M$ steps. Three forms of supervision are considered: base one-step samples, chain-of-thought ($\ct$) samples that reveal full random trajectories, and end-to-end ($\etoe$) samples that reveal only the final token. We study the samples required to learn the one-step probabilities in the base model and the final-token probability in the $\ct$ and $\etoe$ models, under KL divergence.
		
		We show that the exact-accuracy stochastic model admits no universal comparison: both $m_{\ct}/m_{\rm base}$ and $m_{\etoe}/m_{\ct}$ can be simultaneously arbitrarily larger than $M/\varepsilon$. Nevertheless, for every class,
		$m_{\ct}(\varepsilon)\le m_{\rm base}(\varepsilon/M)$ and
		$m_{\etoe}(\varepsilon) = O\left((M/\varepsilon)(1+m_{\ct}(c\varepsilon))\right)$ for a constant $c$,
		at constant confidence. We prove that these scales and the $M/\varepsilon$ overhead are essentially necessary. 
		We complement these general results by studying one canonical class of logistic generators, analogous to linear threshold functions (LTF) for the deterministic case, and obtain both statistical and computational separations between $\ct$ and $\etoe$. We show that at horizon $M=d^2$, $\ct$ learning uses $\widetilde \Theta(d/\varepsilon)$ trajectories whereas $\etoe$ learning requires $\widetilde \Theta(d^2/\varepsilon)$ samples. 
		Second, polynomial-time $\ct$ learning is possible, whereas under the learning parity with noise (LPN) assumption polynomial-time $\etoe$ learning is impossible even at fixed positive KL accuracy.
	\end{abstract}
	\section{Introduction}
	
	Large language models (LLMs) are, on an abstract level, {\em stochastic autoregressive generators}: given a prompt (input string), they assign a next-token distribution depending on the context, sample a token from this distribution, append it to the context, and repeat with the expanded context. Thus, even when the task is
	not inherently iterative, the output is produced by an iterated next-token process.
	This viewpoint is central to modern LLMs~\cite{vaswani2017attention,brown2020language,ouyang2022training}.
	
	A recent line of theoretical work abstracts this process through {\em autoregressive chain-of-thought} learning. This model was introduced in the pioneering work of Joshi, Vardi, Block, Goel, Li, Misiakiewicz, and Srebro~\cite{joshi2025theory}, where a deterministic next-token function is iterated for $M$ steps. The model introduces two supervision settings:
	in the {\em chain-of-thought} ($\ct$) model, the learner observes the whole autoregressive trajectory, while in the {\em end-to-end} ($\etoe$) model, the learner observes only the final output token.
	In both models, the goal is to predict the final output token in a PAC-learning setting. The picture by now is well understood: at a high level, $\ct$ samples reduce the autoregressive task to a problem comparable to the base learning problem of the same function class, whereas $\etoe$ samples can be substantially harder. Recent works give a near complete sample complexity taxonomy relating the base, $\ct$, and $\etoe$ learning problems~\cite{joshi2025theory,hanneke2026sample,doronarad2026online,li2026optimal}.
	
	The deterministic target class in this abstraction misses a fundamental feature of modern language models: randomness is essential for language generation.
	In practice, an LLM does not deterministically choose the next token given the prompt; rather, it specifies a next-token probability distribution. This is not a minor technicality: likelihood training learns probabilities over tokens, sampling turns these probabilities into diverse plausible continuations, and self-consistency aggregates multiple sampled trajectories from the same prompt~\cite{brown2020language,holtzman2019curious,wang2022selfconsistency}.
	
	This also changes the nature of chain-of-thought samples: in the deterministic model, observing a trajectory reveals exact intermediate labels; in a stochastic model, even a full trajectory reveals only one noisy sample from the intermediate probabilities.
	
	Motivated by this, we study a stochastic regression analogue of autoregressive chain-of-thought learning. A generator assigns a probability distribution over the next token at every prompt. Iterating this rule for multiple steps gives a distribution over trajectories. Under forward KL divergence, the target is learning the final-token probability: the stochastic analogue of the deterministic final token. This is exactly the excess expected logarithmic loss on an observed final token. We make no assumption that the target probabilities are bounded away from $0$ or $1$.
	
	\subsection{The model}
	
	This section defines the stochastic autoregressive learning model, the two supervision models $\ct$ and $\etoe$,
	as well as the preliminary tools needed to state our results.
	
	\paragraph{Stochastic autoregressive generators.}
	Let \(\Sigma=\{0,1\}\) be the binary token alphabet, and fix a generation length \(\iter\ge 1\). A
	stochastic next-token generator is a map \(g:\Sigma^\star\to\Delta(\Sigma)\),
	and a stochastic generator class is a set
	\(\cF\subseteq \Delta(\Sigma)^{\Sigma^\star}\), where $\Delta(\Sigma)$ is the set of all distributions over $\Sigma$.
	Given a prompt \(x\in\Sigma^\star\), the \(\iter\)-step stochastic chain of
	thought \(Z_1,\ldots,Z_\iter\) is generated by iteratively sampling the next token $Z_t$ from the distribution specified by $g(x Z_1 \ldots Z_{t-1})$, i.e., the corresponding distribution of the original prompt concatenated by the trajectory prefix obtained so far. Formally,
	\[
	Z_1\sim g(x),
	\qquad
	Z_t\mid Z_1,\ldots,Z_{t-1}\sim g(xZ_1\cdots Z_{t-1})\quad 2\le t\le M.
	\]
	The \(\iter\)-step end-to-end output probability is the probability that the last generated token is $1$:
	\[q_g^{\etoe-\iter}(x):=\Pr_g(Z_\iter=1\mid X=x).\]
	For \(\iter=1\), let \(p_g(s):=q_g^{\etoe-1}(s)=g(s)(1)\) be the one-step probability of output \(1\) at state \(s\). Namely, \(p_g(s)\) is the one-step probability at state \(s\), while
	\(q_g^{\etoe-M}(x)\) is the global final-token probability from prompt \(x\).
	For example, when $M=2$, it holds that
	$$
	q_g^{\etoe-2}(x)
	=
	p_g(x)p_g(x1)+(1-p_g(x))p_g(x0).$$
	
	\paragraph{Supervision models.}
	Fix a distribution $P$ on prompts $\Sigma^\star$.
	
	\begin{itemize}
		\item end-to-end \(\etoe\): the learner receives i.i.d. samples
		\((X_i,Z_{i,\iter})_{i=1}^n\), where \(X_i\sim P\), and
		\(Z_{i,1},\ldots,Z_{i,\iter}\) are generated autoregressively from an unknown
		\(g_\star\in\cF\). Thus, only the final bit $Z_{i,M}$ of each sample $i$ is given to the learner.
		
		\item  chain-of-thoughts \(\ct\):
		the learner receives the full i.i.d. trajectories
		\((X_i,Z_{i,1},\ldots,Z_{i,\iter})_{i=1}^n\).
	\end{itemize}
	In both models, at test time the learner receives only a fresh \(X\sim P\). Thus, \(\ct\) and \(\etoe\) differ only in the supervision available during training:
	a full autoregressive trajectory versus only the final Bernoulli sample.
	
	The goal is to estimate \(q_{g_\star}^{\etoe-\iter}\) under forward KL divergence. This is the stochastic analogue of the deterministic final label~\cite{joshi2025theory}, with the natural regression target for the observed final bit.
	Note that for
	\(M=1\), both supervision models reduce to the base one-step regression problem.
	We give below the general PAC-learning regression formulation stated according to our setting.
	
	\paragraph{KL learning.}
	The goal is to learn the target probability $q_\star$ under forward KL divergence. For a prompt distribution $P$ and a predictor $\widehat q:\Sigma^\star\to[0,1]$, the error is
	\begin{equation}
		\mathbb E_{X\sim P}\KL(q_\star(X)\|\widehat q(X)).
		\label{eq:common-risk}
	\end{equation}
	Here $\KL(p\|q)$ denotes the standard forward KL divergence between Bernoulli laws (formally defined in Section~\ref{sec:prelim}).
	A learner $(\varepsilon,\delta)$-learns a class
	$\mathcal G\subseteq[0,1]^{\Sigma^\star}$ if, for every target
	$q_\star\in\mathcal G$ and prompt distribution $P$, given $n$
	i.i.d.\ training examples under the specified supervision model,
	it outputs $\widehat q$ satisfying
	$\mathbb E_{X\sim P}\KL(q_\star(X)\|\widehat q(X))\le\varepsilon$
	with probability at least $1-\delta$.
	The predictor need not belong to $\mathcal G$.
	
	\paragraph{Sample complexity.}
	For $0<\delta<1$, define, uniformly over any distribution $P$ on $\Sigma^\star$ and~$g_\star\in\cF$:
	\begin{description}
		\item[$m_{\rm base}^{\cF}(\varepsilon,\delta)$:] the least $n$ for which
		$n$ i.i.d. samples $(X_i,Y_i)$ with $X_i\sim P$ and $Y_i\mid X_i\sim\Ber(p_{g_\star}(X_i))$ suffice
		to $(\varepsilon,\delta)$-learn $p_{g_\star}$.
		
		\item[$m_{\ct}^{\cF,M}(\varepsilon,\delta)$:] the least $n$ for which $n$ i.i.d. full $\ct$s $(X_i,Z_{i,1},\ldots,Z_{i,M})$ suffice to $(\varepsilon,\delta)$-learn~$q_{g_\star}^{\etoe-M}$.
		
		\item[$m_{\etoe}^{\cF,M}(\varepsilon,\delta)$:] the least $n$ for which $n$ i.i.d. $\etoe$ samples $(X_i,Z_{i,M})$ suffice to $(\varepsilon,\delta)$-learn $q_{g_\star}^{\etoe-M}$.
	\end{description}
	
	By convention,
	zero-sample learners are allowed. If no finite $n$ satisfies the requirement,
	the sample complexity is $+\infty$.  When $\delta=1/3$, we omit it and simply write $m_{\rm base}^{\cF}(\varepsilon)$, $m_{\ct}^{\cF,M}(\varepsilon)$, and $m_{\etoe}^{\cF,M}(\varepsilon)$.
	We assume henceforth that all function classes satisfy the standard pointwise-measurability convention~\cite{vaartwellner1996weak,vandegeer2000empirical}: Whenever a proof optimizes over an arbitrary real-valued class using the training sample, the same optimization is assumed to be approximable arbitrarily well by optimizing over a countable subcollection. In particular, exact empirical risk minimizers may be replaced by measurable almost-minimizers. This only excludes nonmeasurable pathologies and is automatic for finite and parametric classes. A learner is a total measurable map on finite samples and random seeds; its behavior on samples outside the model may be arbitrary.
	
	\subsection{Our results}
	\label{subsec:our-results}
	Our main motivation is to understand the sample complexity of stochastic autoregressive $\etoe$ learning, both compared to $\ct$ learning and compared to the base one-step class. This is analogous to the deterministic line of work~\cite{joshi2025theory,hanneke2026sample,doronarad2026online}, where sample complexity is governed by the VC dimension~\cite{blumer1989learnability}. Showing that both $\ct,\etoe$ learning can sometimes be much easier than learning the base class is immediate, as there are classes for which one-step probabilities are complex yet every trajectory of $M$ steps always yields the same token.\footnote{For example, allow arbitrary probabilities only at distinct equal-length roots $r=b1$, and set every other transition to $0$. A trajectory cannot enter a root unless it starts there, and after leaving a root it emits only zeros. Hence the final token is always $0$ for $M\ge2$, whereas the base problem can be arbitrarily complex.}
	In addition, for every class $\cF$, $M\ge1$, and $\varepsilon,\delta\in(0,1)$,
	$m_{\ct}^{\cF,M}(\varepsilon,\delta)\le m_{\etoe}^{\cF,M}(\varepsilon,\delta)$,
	since a $\ct$ learner can ignore the intermediate tokens and use only the final-token observations.
	
	Can learning $\ct$ be as easy as learning the base class, as in the deterministic setting? How much harder can $\etoe$ learning be with respect to $\ct$ learning and the base class? We give answers to these questions below. Then, we study the class of logistic generators in the stochastic autoregressive model and show what can be done for each of the supervision models $\ct$ and $\etoe$.
	
	\subsection*{General sample-complexity comparisons}
	In the original deterministic prediction problem, the sample complexity of learning with $\ct$ supervision is independent of $M$ for any VC-class~\cite{hanneke2026sample}, while the analogous quantity without $\ct$ supervision can grow linearly with $M$~\cite{joshi2025theory}. One might have expected a similar exact-scale comparison for the stochastic model. Our first result shows that the exact-accuracy stochastic model has no universal comparison between base, $\ct$, and $\etoe$ learning. For the same class, both ratios $m_{\ct}/m_{\rm base}$ and $m_{\etoe}/m_{\ct}$ can be made arbitrarily larger than $M/\varepsilon$.
	
	\begin{informalresult}{Theorem}{thm:simultaneous-exact-scale-taxonomy}
		For all sufficiently large $M$, all sufficiently small $\varepsilon>0$, and every $A\ge1$, there is a finite stochastic generator class $\cF$ with $0<m_{\rm base}^{\cF}(\varepsilon)<\infty$ such that
		\[
		\frac{m_{\ct}^{\cF,M}(\varepsilon)}{m_{\rm base}^{\cF}(\varepsilon)}
		>A\frac M\varepsilon,
		\qquad
		\frac{m_{\etoe}^{\cF,M}(\varepsilon)}{m_{\ct}^{\cF,M}(\varepsilon)}
		>A\frac M\varepsilon.
		\]
	\end{informalresult}
	Thus, at the exact same accuracy, neither one-step learning controls chain-of-thought learning nor chain-of-thought learning controls end-to-end learning. Meaningful universal comparisons require different scales. These are worst-case examples whose classes depend on the requested accuracy, not persistent separations claimed for every natural class. Our general bounds considering the three supervision models are summarized in Table~\ref{fig:table_1}. Unless displayed, the confidence is $2/3$; the formal results also gives arbitrary-confidence bounds.
	
	\begin{table}[tb]
		\centering
		\renewcommand{\arraystretch}{1.6}
		\begin{tabular}{@{\hspace{1.2em}}c@{\hspace{1.2em}}}
			\toprule
			Main General Results \\
			\midrule
			$\displaystyle m_{\ct}(\varepsilon)/m_{\rm base}(\varepsilon),
			m_{\etoe}(\varepsilon)/m_{\ct}(\varepsilon) = \omega(M/\varepsilon)$ \\
			$\displaystyle m_{\ct}(\varepsilon)\le m_{\rm base}(\varepsilon/M)$ \\
			$\displaystyle m_{\etoe}(\varepsilon)\le
			C\frac{M}{\varepsilon}\bigl(1+m_{\ct}(c\varepsilon)\bigr)$ \\
			\bottomrule
		\end{tabular}
		\caption{General sample-complexity comparisons at constant confidence. The first line is an existence statement, simultaneously for both ratios in one finite class. The other lines hold for every class; $c,C>0$ are universal.}
		\label{fig:table_1}
	\end{table}
	
	We first upper bound the sample complexity of $\ct$ at scale $\varepsilon$ by base learning at scale $\varepsilon/M$. This holds for every stochastic class. Some finite classes have both sample complexities equal up to constant factors after multiplying the base error $\varepsilon/M$ by a fixed constant. For others, allowing base error $\varepsilon^{1-\beta}/M$, for any fixed $0<\beta<1$, permits zero base samples while $\ct$ still needs arbitrarily many samples.
	
	\begin{informalresult}{Theorem}{thm:base-to-cot-full}
		For every class $\cF$, $M\ge1$, $\varepsilon>0$, and $0<\delta<1$,
		\[
		m_{\ct}^{\cF,M}(\varepsilon,\delta)
		\le m_{\rm base}^{\cF}(\varepsilon/M,\delta).
		\]
		Moreover, there is a universal constant $\kappa>0$ such that finite classes attain
		$m_{\ct}(\varepsilon)=\Theta(LM/\varepsilon)$ and $m_{\rm base}(\kappa\varepsilon/M)=\Theta(LM/\varepsilon)$
		for each $L\ge1$, $M\ge2$, and all sufficiently small $\varepsilon>0$.
		For each $0<\beta<1$ and $K\ge1$, and sufficiently small $\varepsilon$, other finite classes satisfy
		$m_{\rm base}(\varepsilon^{1-\beta}/M)=0$
		but $m_{\ct}(\varepsilon)\ge cK/\varepsilon$.
	\end{informalresult}%The matching lower bounds are proved in Theorems~\ref{thm:base-match} and~\ref{thm:base-scale-lower}.
	
	Interestingly, for comparing $\ct$ and $\etoe$, all that is needed is a constant change in accuracy. The $M/\varepsilon$ overhead is attainable even when $\ct$ is evaluated at a fixed smaller multiple of the requested $\etoe$ accuracy.
	
	\begin{informalresult}{Theorem}{thm:main_2}
		For every class $\cF$, every $M\ge1$, and $0<\varepsilon\le1$,
		\[
		m_{\etoe}^{\cF,M}(\varepsilon)
		\le C\frac{M}{\varepsilon}
		\bigl(1+m_{\ct}^{\cF,M}(c\varepsilon)\bigr).
		\]
		For every fixed $0<\rho\le1$, every $N\ge1$, every $M\ge2$, and all sufficiently small $\varepsilon$, some finite class satisfies
		$
		m_{\ct}^{\cF,M}(\rho\varepsilon)=\Theta_\rho(N)$ and
		$m_{\etoe}^{\cF,M}(\varepsilon)=\Theta(NM/\varepsilon).
		$
	\end{informalresult}
	As a corollary, $m_{\etoe}(\varepsilon)\le
	C(M/\varepsilon)(1+m_{\rm base}(c\varepsilon/M))$.
	The universal constant-confidence comparison has no hidden logarithmic factors. At arbitrary confidence, only logarithmic overhead is needed (Theorem~\ref{thm:kl-comparison-confidence}).
	
	\subsection*{Logistic case study}
	The above results give a detailed picture on general worst-case classes. We choose one canonical class and study its stochastic autoregressive sample complexity under $\ct$ and $\etoe$ supervision. For $s\in\{0,1\}^\star$, let $\operatorname{tail}_d(s)\in\{0,1\}^d$ be the last $d$ bits of $s$, padded by zeros on the left if $|s|<d$, and let $\sigma(a)=1/(1+e^{-a})$ be the sigmoid, or logistic, function. Define
	\begin{equation}
		\mathcal F_\sigma(d):=
		\left\{g_{\beta,w}\ \middle|\ (\beta,w)\in\mathbb R\times\mathbb R^d,\quad
		p_{g_{\beta,w}}(s)=\sigma\bigl(\beta+\langle w,\operatorname{tail}_d(s)\rangle\bigr)\right\}.
		\label{eq:affine-logistic}
	\end{equation}
	That is, the same logistic rule is reused at every step, depending only on the last $d$ tokens and the hidden parameters. This class is a finite-memory simplification of the next-token mechanism used by modern autoregressive language models~\cite{vaswani2017attention,brown2020language,ouyang2022training}. Here $d$ denotes memory length; there are $d+1$ real parameters.   No parameter bound is assumed.
	
	We show that the two forms of supervision $\ct,\etoe$ differ substantially both statistically and computationally. We adapt a result of Dewasurendra~\cite{dewasurendra2026fulltraces}, which improves the bound from an earlier version of this work, to obtain a nearly linear $\ct$ bound in $d$. We also obtain an $\etoe$ lower bound using an adaptation of a a result on  linear threshold functions from a recent paper by the authors~\cite{companioncapacity} to logistic KL learning. We also prove an $\etoe$ upper bound by constructing a finite KL cover of the unbounded logistic class; this argument does not follow from those two bounds. These results give the following theorem. Table~\ref{table:2} distinguishes sample complexity from computation.
	
	\begin{table}[tb]
		\centering
		\renewcommand{\arraystretch}{1.4}
		\begin{tabular}{@{}lcc@{}}
			\toprule
			& $\ct$ & $\etoe$ \\
			\midrule
			Sample complexity & $\widetilde O(d/\varepsilon)$ & $\widetilde\Theta(d^2/\varepsilon)$ \\
			Polynomial time learners & Yes & No (under LPN) \\
			\bottomrule
		\end{tabular}
		\caption{Logistic separations between $\ct$ and $\etoe$ supervision at constant confidence. Sample bounds need not be attained efficiently; hardness holds at fixed KL accuracy with $M\le d$. Sample complexity holds for horizon $M=d^2$.}
		\label{table:2}
	\end{table}
	
	\begin{informalresult}{Theorem}{thm:logistic-statistical-separation}
		For sufficiently large $d$ and sufficiently small $\varepsilon>0$, at the common horizon $M=d^2$ we have
		$
		m_{\ct}^{\mathcal F_\sigma(d),d^2}(\varepsilon)
		\le\widetilde O(d/\varepsilon)$ and 
		$m_{\etoe}^{\mathcal F_\sigma(d),d^2}(\varepsilon)
		=\widetilde\Theta(d^2/\varepsilon).
		$
	\end{informalresult}
	The $\ct$ upper bound holds for all horizons with only logarithmic horizon dependence. The $\etoe$ lower bound holds for every $M\ge d^2$; our matching-order upper bound at polynomial horizons is given in Theorem~\ref{thm:tail-linear-sigmoid-e2e-upper}.
	
	\begin{informalresult}{Theorem}{thm:lpn-e2e-hardness-tied-sigmoid}
		Assume the \textnormal{LPN} prediction assumption. Then no randomized polynomial-time algorithm can learn $\mathcal F_\sigma(d)$ from $\etoe$ samples for all $M\le d$, even at fixed positive KL accuracy and confidence. This holds for arbitrary efficiently evaluable or samplable predictors. In contrast, polynomial-time learning and final prediction from full $\ct$ supervision are possible for all $d,M$ (Corollary~\ref{cor:efficient-kl}).
	\end{informalresult}
	These are separate advantages under the same final-token loss: the cited nearly linear statistical guarantee does not establish efficient attainment of that rate.
	\subsection{Related work}
	\label{sec:related_work}
	
	The closest prior work is the deterministic autoregressive PAC model of
	Joshi et al.~\cite{joshi2025theory} as discussed in the introduction. Hanneke et al.~\cite{hanneke2026sample} gave a
	nearly complete deterministic sample-complexity taxonomy for this model, and the model was recently generalized to online learning~\cite{doronarad2026online} and multiple thinkers~\cite{joshi2026multiplethinkers} (with potentially different $\ct$s but the same final output). Our work differs from all of the above works by studying the stochastic variant of the autoregressive $\ct$ model.
	A different statistical theory of $\ct$ supervision was developed by
	Altabaa, Montasser, and Lafferty~\cite{altabaa2025CoTinfo}. Their model treats
	$\ct$ as an abstract auxiliary supervision signal, rather than as a trajectory
	generated by an autoregressive next-token process. Chain-of-thought supervision is also conceptually related to learning using
	privileged information~\cite{vapnik2009privileged}, since intermediate states
	are observed during training but not at test time.
	
	Foster et al.~\cite{foster2024behavior} study a more general model called imitation learning that includes autoregressive generation as a special case. Their Corollary C.4 gives a trajectory Hellinger guarantee with linear dimension dependence for logistic generators with bounded weights and feature norms (here we do not make these assumptions). For finite $\cF$, both full-trajectory $\ct$ learning and final-token
$\etoe$ learning achieve expected forward KL at most
$\log(|\cF|+1)/n$ (Proposition~\ref{prop:trajectory}).
Thus, our separations do not contradict~\cite{foster2024behavior}:
they compare different supervision or accuracy scales, and may use
classes whose size grows with $M$.
    
    Improving upon the earlier version of our work~\cite{doronarad2026stochasticv1}, Dewasurendra~\cite{dewasurendra2026fulltraces} obtained a nearly linear, norm-free sample complexity for $\ct$ supervision using conditional $\rho$-estimation~\cite{baraudbirge2018rho}. Other recent works~\cite{rohatgi2025nexttoken,xu2026autoregressive} study learning the full autoregressive sequence distribution from complete trajectories when the true generator does not necessarily belong to the model class. 
	In addition,
	Rajaraman et al.~\cite{rajaraman2026autocurriculum} study stochastic
autoregression with adaptive teacher queries and an outcome verifier,
aiming to reduce demonstration complexity. In contrast, we learn the
final-token probabilities in forward KL from i.i.d. samples,
without a verifier.
	
	% A recent paper by the authors~\cite{companioncapacity} studies finite-memory autoregressive threshold functions. Its quadratic threshold VC-dimension lower bound is a cited input here; our additional argument transfers deterministic shattering to finite logistic parameters and a forward-KL sample lower bound.
	
	Intermediate reasoning traces have been studied extensively in empirical work~\cite{nye2021scratchpads,wei2022chainofthought,kojima2022zeroshot,zelikman2022star,wang2022selfconsistency,yao2023tree,uesato2022process,lightman2023verify}.
	A separate theoretical line studies their computational role in concrete models,
	especially transformers and simple linear predictors trained on chain-of-thought data
	\cite{merrill2023expressive,li2024chain,wen2025sparse,malach2024autoregressive,yang2026tight}.
	Magen, Joshi, and Vardi~\cite{magen2026single} show that a single deterministic autoregressive neuron can simulate feedforward circuits with polynomial overhead, whereas we study stochastic generation and forward-KL learning.
	
	The one-step base class learning problem in our model is learning a Bernoulli
	conditional distribution in forward KL: given samples
	\(Y\mid X=x\sim \operatorname{Bern}(p(x))\), learn \(p(x)\) distribution-free.
	This problem was studied in probabilistic concept learning~\cite{kearns1994pconcepts}
	and is covered by decision-theoretic PAC learning~\cite{haussler1992decision}.
	The scale-sensitive theory of real-valued learning~\cite{bartlett1996fat,alon1997scale,anthonybartlett1999neural}, and its entropy bounds~\cite{mendelsonvershynin2003entropy,rudelson2006combinatorics}, provide further background. Here the general comparisons are stated directly through sample complexities, rather than through shattering dimensions.
	PAC and generalization guarantees have also been studied for probabilistic and
	weighted automata~\cite{clark2004pac,balle2013learning}. These works learn
	automaton-based distributions or string functions. Finally, our computational lower bound for end-to-end learning is based
	on noisy parity, in line with standard statistical query and cryptographic hardness tools in
	PAC learning~\cite{kearns1993sq,blumkalaiwasserman2003,kearnsvaliant1994crypto}.

	\paragraph{Organization.}
	Section~\ref{sec:prelim} defines KL and notation. Section~\ref{sec:proof-sketches} gives proof sketches, and Section~\ref{sec:tools} states the statistical tools, proved in Appendix~\ref{sec:tool-proofs}. Section~\ref{sec:comparisons} proves the general comparisons, and Section~\ref{sec:sharp} constructs classes attaining the bounds. Section~\ref{sec:same_scale} proves the simultaneous exact-scale separation. Section~\ref{sec:logistic} proves the logistic statistical bounds and computational separation. Section~\ref{sec:future-work} discusses the scope of the results.

    \section{Preliminaries}\label{sec:prelim}

	\paragraph{KL divergence.}
	For probability laws $Q_\star,Q$ on a finite or countable space, define their forward KL divergence by
	\[
	\KL(Q_\star\|Q)=\sum_z Q_\star(z)\log\frac{Q_\star(z)}{Q(z)}.
	\]
	We use natural logarithms, set a summand with $Q_\star(z)=0$ to zero, and set the divergence to $+\infty$ when $Q(z)=0<Q_\star(z)$. For Bernoulli laws we abbreviate
	\[
	\KL(p\|q):=\KL(\Ber(p)\|\Ber(q))
	=p\log\frac pq+(1-p)\log\frac{1-p}{1-q}.
	\]
	The first argument is always the target and the second the prediction.

    \paragraph{Further definitions and notations.}
	For a generator $g$, let $P_g^M(\cdot\mid x)$ be the law of its
	$M$ generated bits from prompt $x$.
	For a fixed target $g_\star$, write $q_\star=q_{g_\star}^{\etoe-M}$.
	All errors use forward KL as in~\eqref{eq:common-risk}.
	For an integer $k\ge1$, write $[k]=\{1,\ldots,k\}$, and
	let $\mathcal L(V)$ denote the law of a random variable $V$.
	Write $z_{<t}=z_1\cdots z_{t-1}$, with $z_{<1}$ the empty string,
	and $\mathbf1\{E\}$ for the indicator of an event $E$.
	A dyadic rational is $z/2^b$, where $z$ is an integer and $b\ge0$
	is an integer. We sometimes use unsubscripted $c,C>0$ to denote universal constants.
	Let $\operatorname{TV}(Q,R)=\frac12\sum_z|Q(z)-R(z)|$; Pinsker's inequality gives
	$\operatorname{TV}(Q,R)^2\le\KL(Q\|R)/2$.
	 We also use the KL chain rule and the Bernoulli bounds
	\begin{equation}\label{eq:bernoulli-bounds}
		2(u-v)^2\le\KL(u\|v)\le\frac{(u-v)^2}{v(1-v)}
		\qquad(u,v\in[0,1],\ 0<v<1).
	\end{equation}
	The first is Pinsker's inequality, and the second follows from
	$\log t\le t-1$. The bounds also apply at $u=0,1$.
	
	\section{Proof sketches}\label{sec:proof-sketches}
	We outline the main arguments of our proofs.
	For the lower-bound constructions, we use a fixed finite set $R=\{r_1,\ldots,r_K\}$
	of distinct equal-length binary prompts, called \emph{roots}. A descendant $r_jz$ of root $r_j$
	is the concatenation of $r_j$ and $z\in\Sigma^\star$, at depth $|z|$.
	The generators we construct may have different next-token probabilities only at strings
	$r_jz$ with $|z|<M$. At every other string, all generators use the same
	next-token probability.
	Lemma~\ref{lem:roots} extends these rules to all strings so that
	$q_g^{\etoe-M}(x)$ is the same for every generator whenever $x\notin R$.
	Thus, only the final means at the roots need to be learned, even for arbitrary prompt laws.
	
	\paragraph{Theorem~\ref{thm:simultaneous-exact-scale-taxonomy}.}
	We combine two finite classes on disjoint roots to separate both
    $\ct$ from base learning and $\etoe$ from $\ct$ learning.
	The first class, $G_N$, contains generators indexed by signs $\omega\in\{-1,1\}^N$ and $\tau\in\{-1,1\}$.
	At each root $r_j$, the generator $g_{\omega,\tau}$ samples the
first $J$ bits independently from $\Ber(1/2+t_G\omega_j)$,
samples the remaining intermediate bits as fair coins,
and sets the final bit to the majority of the first $J$ bits.
	Choose a fixed odd $J$ and $t_G=\Theta(\sqrt\varepsilon)$ so that
	$\KL(1/2\pm t_G\|1/2)\le\varepsilon/100$ and
	$q_g^{\etoe-M}(r_j)=1/2+8\sqrt\varepsilon\,\omega_j$.
	Base (one-step) learning can completely ignore $\omega$, whereas learning its $N$ signs from weakly
	biased trajectory bits requires $\Theta(N/\varepsilon)$ CoT samples.
	An additional root $v$ has first-bit mean $1/2+4\tau\sqrt\varepsilon$
	and fair descendants, forcing $\Theta(1/\varepsilon)$ base samples.

The second class $H_D=\{h_0\}\cup
\{h_{b,s}:b\in\{0,1\},\,s\in\{-1,1\}^D\}$
uses roots $x_0,\ldots,x_D$.
The inactive generator $h_0$ emits $0$ followed by fair bits.
Each active $h_{b,s}$ begins with $1$ then $b$ and has final means
$
q_{h_{b,s}}^{\etoe-M}(x_0)=b$ and
$q_{h_{b,s}}^{\etoe-M}(x_j)=1/2+h_\varepsilon s_j$ for
$j\in[D]$ 
where $\KL(1/2\pm h_\varepsilon\|1/2)=\varepsilon$.
The final means at $x_1,\ldots,x_D$ arise from majorities
of weakly biased bits, allowing base learning to ignore
the signs $s$ and use only $O(1/\varepsilon)$ samples.
A $\ct$ trace reveals whether the generator is active and, if so, $b$.
Thus, $O(1/\varepsilon)$ traces suffice here: if the roots have total
mass at most $\varepsilon/\log 2$, predicting the inactive law
already incurs KL error at most $\varepsilon$.
For $\etoe$ learning with $n$ samples, set probability masses
$P(x_0)=\eta\ll1/(n+1)$ and $P(x_j)=(1-\eta)/D$.
With probability at least $1-n\eta$, the samples reveal
nothing about $b$. For sufficiently large $D$, the unseen signs
contribute KL loss at least $\varepsilon-\eta\log 2/2$
with high probability, while the unknown $b$ adds
at least $\eta\log 2$ with probability $1/2$.
Thus, $n$ $\etoe$ samples do not suffice.
    
% 	The second class is $H_D=\{h_0\}\cup\{h_{b,s}:b\in\{0,1\},\ s\in\{-1,1\}^D\}$,
% 	on roots $x_0,\ldots,x_D$. At every root, $h_0$ emits $0$ and then fair bits;
% 	$h_{b,s}$ emits $1,b$ as its first two bits. Its final means are
% 	\[
% 	q_{h_{b,s}}^{\etoe-M}(x_0)=b,
% 	\qquad q_{h_{b,s}}^{\etoe-M}(x_j)=1/2+h_\varepsilon s_j\quad(j\in[D]),
% 	\]
% 	where $\KL(1/2\pm h_\varepsilon\|1/2)=\varepsilon$.
% 	The latter means again come from majorities of weak coins, keeping base cost
% 	$O(1/\varepsilon)$. A trace (chain-of-thought) whose first bit is $1$ identifies $h_{b,s}$ as the
% 	target type and reveals $b$; predicting $b$ at $x_0$ and $1/2$ at the other roots
% 	then costs at most $\varepsilon$. If these roots have probability mass $>\varepsilon/\log 2$, then
% $O(1/\varepsilon)$ trajectories reveal an active trace with high probability;
% otherwise, the inactive predictor has KL error at most $\varepsilon$.
% 	For an end-to-end sample size $n$, set the probabilities $P(x_0)=\eta\ll1/(n+1)$ and
% 	$P(x_j)=(1-\eta)/D$ of choosing these prompts. Only the final token at $x_0$ depends on $b$,
% with probability at least $1-n\eta$ the samples
% reveal no information about $b$.
% 	We clip predictions at $x_1,\ldots,x_D$ to $[1/2-h_\varepsilon,1/2+h_\varepsilon]$,
% 	 so they cannot increase the loss. The unseen signs then contribute expected loss at least
% 	$(1-\eta)(1-n/D)\varepsilon$, with variance $O(\varepsilon^2/D)$.
% 	Their loss therefore exceeds $\varepsilon-\eta\log2/2$ with probability near one.
% 	The unknown bit $b$ adds at least $\eta\log2$ with probability $1/2$, so learning fails.
	
	Fix $g_0\in G_N$ and combine the rules on their disjoint root trees,
allowing only pairs $(g,h_0)$ and $(g_0,h)$
for $g\in G_N$ and $h\in H_D$.
	Observing a first bit $1$ at an $H_D$ root then identifies the $G_N$ rule as $g_0$.
	The combined class has base cost $\Theta(1/\varepsilon)$ and CoT cost
	$\Theta(N/\varepsilon)$, while its $\etoe$ cost grows without bound with $D$.
	Choose $N$ first and $D$ second to obtain both ratios.
	
	\paragraph{Theorem~\ref{thm:base-to-cot-full}.}
    We reduce $\ct$ to base learning as follows. 
	For each trajectory $(X,Z_1,\ldots,Z_M)$, choose $U$ uniformly in $[M]$
	and give $(S,Y)=(XZ_{<U},Z_U)$ to the base learner. Let $\mu$ be the distribution of these inputs $S$.
Then, $Y\mid S\sim\Ber(p_{g_\star}(S))$, where $g_{\star}$ is the true generator.
	Let $\widehat p$ be the predictor output by the base learner and define $p_{\widehat g}=\widehat p$.
	Using the KL chain rule we have
	\[
	\mathbb E_X\KL\left(P_{g_\star}^M(\cdot\mid X)\|P_{\widehat g}^M(\cdot\mid X)\right)
	=M\mathbb E_{S\sim\mu}\KL(p_{g_\star}(S)\|\widehat p(S)).
	\]
	Projecting a trajectory to its final bit cannot increase KL, so base error
	$\varepsilon/M$ gives final-token error~$\varepsilon$.
	To show this scale is necessary, use $K$ roots with independent signs and
	probabilities $p_j=1/M\pm a\sqrt\varepsilon/M$, for a fixed large $a$,
	until the first $1$ is obtained, then continue to generate only $1$s. The final mean $1-(1-p_j)^M$
	has derivative $\Theta(M)$ near $1/M$.
	Predicting $1/M$ costs $O(\varepsilon/M)$ in base KL, but the final means
	differ by $\Theta(\sqrt\varepsilon)$ and require $\Omega(K/\varepsilon)$ $\ct$ samples.
	For showing a matching sample complexity, take $K=\lceil LM\rceil$, $L\ge1$, and add
	$\lceil L\rceil$ roots with first-bit means $1/2\pm b\sqrt{\varepsilon/M}$
	and fair descendants, for a fixed large $b$. These force $\Theta(LM/\varepsilon)$
	base samples at error $\kappa\varepsilon/M$, without changing the final-token problem.
	
	\paragraph{Theorem~\ref{thm:main_2}.}
	    We reduce $\etoe$ to $\ct$ learning as follows. 
    Let $n=m_{\ct}(\varepsilon/24)$. On $n$ prompts from $\etoe$ samples, 
    run a $\ct$ learner on every possible assignment of the $nM$
trajectory bits, using the same random seed for the $\ct$ learner.
	These $2^{nM}$ runs yield a list of predictors. One assignment
matches all $n$ actual, unobserved trajectories, so the list contains
the predictor the $\ct$ learner would output on those trajectories.
	Repeat this process independently three times, reducing the failure probability. Using $O((1+nM)/\varepsilon)$ additional $\etoe$ samples,
an aggregation algorithm combines the resulting predictions
into one predictor with KL error at most $\varepsilon$,
proving the upper bound.
    
	To show a class reaching this bound, arrange $N(M-1)$ roots as cells $(j,i)$ with
	$j\in[N]$, $i\in[M-1]$, and assign independent signs $\omega_{j,i}\in\{-1,1\}$.
	A trace from cell $(j,i)$ first generates independent bits $Z_t$ of means
	$1/2+\omega_{j,t}/4$ for all $t\in[M-1]$. Its last bit has conditional mean
	$1/2+16\sqrt\varepsilon(2Z_i-1)$, hence final mean
	$1/2+8\sqrt\varepsilon\,\omega_{j,i}$ (illustrated in Figure~\ref{fig:row-construction}).
	Each CoT sample observes all signs in one row at constant bias, giving
	expected loss $O(\varepsilon N/n)$ from $n$ samples for arbitrary prompt laws.
	An $\etoe$ sample observes only its cell's weakly biased bit; testing
	all $N(M-1)$ signs requires $\Theta(NM/\varepsilon)$ samples.
	
	\begin{figure}[tb]
		\centering
		\begin{tikzpicture}[scale=.92,>=Stealth,every node/.style={font=\small}]
			\foreach \j/\lab in {0/{1},1/{2},2/{\cdots},3/{N}} {
				\foreach \i in {0,1,2,3} {
					\node[circle,draw,inner sep=1.4pt] (n\j\i) at (1.15*\i,-.75*\j) {};
				}
				\draw[->] (n\j0)--(n\j1);
				\draw[->] (n\j1)--(n\j2);
				\draw[->] (n\j2)--(n\j3);
				\node[left] at (-.35,-.75*\j) {$\lab$};
			}
			\foreach \i/\lab in {0/{1},1/{2},2/{\cdots},3/{M-1}} {
				\draw[dashed] (1.15*\i,.3)--(1.15*\i,-2.55);
				\node[above] at (1.15*\i,.3) {$\lab$};
			}
			\node[below right] at (1.15,-.75) {$\omega_{j,i}$};
		\end{tikzpicture}
		\caption{The grid construction for the lower bound in Theorem~\ref{thm:main_2}. A $\ct$ trajectory from row $j$ gives one constant-bias observation of every sign in that row. An $\etoe$ sample gives only one weakly biased bit at its own cell. Arrows indicate the signs observed along a trace, not transitions between root prompts.}
		\label{fig:row-construction}
	\end{figure}
	
	\paragraph{Theorem~\ref{thm:logistic-statistical-separation}.}
	We establish quadratic $\etoe$ and nearly linear $\ct$ rates for logistic generators. The $\ct$ upper bound follows by smoothing the full-trace estimator
	of Dewasurendra~\cite{dewasurendra2026fulltraces}.
	For our $\etoe$ upper bound, we impose linear constraints
that preserve large positive and negative linear scores (logits)
and approximate the remaining scores on all Boolean states. Cramer's rule gives a rational solution with bounded coefficients.
Rounding these coefficients yields a finite grid $V$ with
$\log|V|=O(d^2\log(d+1)+d\log(M/\varepsilon))$,
such that every generator has a grid approximation with
one-step KL error $O(\varepsilon/M)$ on every Boolean state.
The KL chain rule gives final-token error $O(\varepsilon)$,
and aggregation over $V$ proves the upper bound.
    
	For the lower bound, we start with $D+1=\Omega(d^2)$ prompts shattered by linear threshold functions (LTFs)~\cite{companioncapacity}. Perturb the bias to remove zero margins and
	scale the weights, so logistic trajectories agree with the threshold trajectories
	with arbitrarily high probability. Give $D$ prompts independent binary labels
	and probability $256\varepsilon/D$ each; give one fixed-label prompt the remaining
	probability $1-256\varepsilon$. With $o(D/\varepsilon)$ samples, a constant fraction
	of the independent labels remain unknown, forcing KL risk above $\varepsilon$.
	
	\paragraph{Theorem~\ref{thm:lpn-e2e-hardness-tied-sigmoid}.}
	We sketch the computational separation for logistic learning. For CoT supervision, note that each observed prefix determines the feature vector
	used to generate the next observed bit. Thus, the empirical negative log likelihood
	is convex in the shared parameters. Every logistic generator admits a bounded weight approximation
with small trajectory error. Minimizing the empirical negative
log-likelihood over these bounded weights gives an efficient
trajectory estimator.
	To predict the final-token probability, simulate the learned
generator several times using random bits fixed in advance,
and average the resulting final bits. A union bound over the $2^d$ possible tails gives uniform accuracy;
	smoothing the simulated final means gives the KL guarantee.
	
	For $\etoe$ hardness, start with noisy parity labels
	$y=\langle a,x\rangle_{\mathbb F_2}\oplus\xi$, where $a\in\{0,1\}^m$ is unknown,
	$x$ is uniform, and $\xi\sim\Ber(\eta)$ is independent, with fixed $0<\eta<1/2$.
	Encode copies of $x$ in separated prompt blocks. Each weight reads the bit at a fixed distance before the current
position. We place copies of the input bits in the prompt so that
each weight reads its assigned bit at one designated step and
zero at all other steps. Large weights on these prompt bits dominate the
	previously generated bits and produce
	$T_k=\mathbf1\{\sum_i a_ix_i\ge k\}$ at steps $k\in[m]$, with exponentially small error.
	At step $m+1$, a fixed prompt bit supplies the bias,
while the first $m$ generated bits contribute an alternating sum approximating
$\sum_{k=1}^m(-1)^{k+1}T_k
=(\sum_i a_ix_i)\bmod 2$.
	The final logistic mean is close to $\eta$ for even parity and $1-\eta$ for odd parity,
	using polynomial-size rational weights in dimension $d=O(m^3)$.
	An efficient $\etoe$ learner trained on LPN samples,
with each input replaced by its encoded prompt and its label unchanged,
would estimate the noisy parity probabilities on most inputs. Repeated evaluation and thresholding at $1/2$ would predict
	noiseless parity with constant advantage, contradicting LPN.

	\section{Statistical tools}\label{sec:tools}
	\begingroup
	\linespread{0.97}\selectfont
	\setlength{\abovedisplayskip}{5pt plus 1pt minus 1pt}
	\setlength{\belowdisplayskip}{5pt plus 1pt minus 1pt}
	\setlength{\abovedisplayshortskip}{2pt plus 1pt}
	\setlength{\belowdisplayshortskip}{3pt plus 1pt minus 1pt}
	
	We state three auxiliary results used below; their proofs appear in
	Appendix~\ref{sec:tool-proofs}. 
	
	\begin{lemma}\label{lem:aggregation}
        Let $P$ be a prompt distribution, $p_\star:\Sigma^\star\to[0,1]$,
        $X\sim P$, and $Y\mid X\sim\Ber(p_\star(X))$.
        For integers $K,T\ge1$ and predictors $q_1,\ldots,q_K:\Sigma^\star\to[0,1]$,
        $T$ independent copies of $(X,Y)$ suffice to output $\bar q$ with
        \begin{equation}\label{eq:aggregation}
            \mathbb E\!\left[\mathbb E_{X\sim P}\KL(p_\star(X)\|\bar q(X))\right]
            \le \min_{j\in[K]}\mathbb E_{X\sim P}\KL(p_\star(X)\|q_j(X))
            +\frac{\log(K+1)}{T}.
        \end{equation}
        The outer expectation is over the samples and learner randomness.
        If $q_1,\ldots,q_K$ are random and independent of the $T$ samples,
the bound holds conditional on their realized values.
    \end{lemma}

Applying Lemma~\ref{lem:aggregation} gives the following bounds for finite generator classes.
    \begin{proposition}\label{prop:trajectory}\label{cor:finite-final}
        For every finite generator class $\cF$, integers $M,n\ge1$,
        and $b\in\{\ct,\etoe\}$, a learner using $n$ samples
        with supervision $b$ outputs $\widehat q$ such that,
        for every $g_\star\in\cF$ and prompt law $P$,
        \[
        \mathbb E\!\left[\mathbb E_{X\sim P}\KL(q_\star(X)\|\widehat q(X))\right]
        \le\frac{\log(|\cF|+1)}n.
        \]
        Given $n$ i.i.d.\ $\ct$ samples, a learner also outputs
        a full-trajectory conditional law $\widehat Q^M$ satisfying
        the same bound for
        $\mathbb E[\mathbb E_{X\sim P}\KL(P_{g_\star}^M(\cdot\mid X)\|\widehat Q^M(\cdot\mid X))]$.
        In particular, for every $\varepsilon>0$ and $b\in\{\ct,\etoe\}$ it holds that
        $m_b^{\cF,M}(\varepsilon)\le\lceil3\log(|\cF|+1)/\varepsilon\rceil$.
    \end{proposition}

We next show how to select a predictor with small KL error
from given predictors $q_1,\ldots,q_K$.
We use independent validation examples $(X,Y)$, drawn from
the target distribution independently of these predictors.
    \begin{lemma}\label{lem:kl-selection}
        Let $P$ be a distribution on $\Sigma^\star$,
        $p_\star:\Sigma^\star\to[0,1]$, $X\sim P$,
        $Y\mid X\sim\Ber(p_\star(X))$, $K\ge1$,
        $0<\varepsilon\le1$, and $0<\delta<1$. Suppose
        $q_1,\ldots,q_K:\Sigma^\star\to[0,1]$ are independent of
        the validation data and satisfy
        $\min_j\mathbb E_{X\sim P}\KL(p_\star(X)\|q_j(X))\le\varepsilon/24$.
        Set $a=\varepsilon/128$ and $B=\log(256/\varepsilon)$.
        Empirical log-loss minimization over $(1-a)q_j+a/2$ has KL
        error at most $\varepsilon$, with probability $1-\delta$, using
        $T\ge (20B/\varepsilon)\log(K/\delta)$ validation examples.
    \end{lemma}

    \section{General sample-complexity comparisons}\label{sec:comparisons}
	We first reduce CoT learning to base learning, then end-to-end learning to CoT learning.
	
	\subsection{From base to chain-of-thought supervision}
	\begin{theorem}\label{thm:base-to-cot-full}
		For every $\cF$, $M\ge1$, $\varepsilon>0$, and $0<\delta<1$,
		\begin{equation}\label{eq:base-cot}
			m_{\ct}^{\cF,M}(\varepsilon,\delta)
			\le m_{\base}^{\cF}(\varepsilon/M,\delta).
		\end{equation}
		The reduction also learns the full conditional trajectory law in KL
		with accuracy $\varepsilon$.
	\end{theorem}
	\begin{proof}\linespread{0.94}\selectfont
		For each i.i.d.\ full trajectory $(X,Z_1,\ldots,Z_M)$, draw
		$U\sim\operatorname{Unif}([M])$ independently and form $(S,Y)=(XZ_{<U},Z_U)$.
		The resulting pairs are i.i.d., with
		$S\sim\mu:=M^{-1}\sum_{t=1}^M\mathcal L(XZ_{<t})$
		and $Y\mid S\sim\Ber(p_{g_\star}(S))$.
		Run a base learner at accuracy $\varepsilon/M$ and confidence $1-\delta$; denote its
		probability predictor by $\widehat p$. Use $\widehat p$ at every
		transition to define $\widehat g$. On the base learner's success event,
		the KL chain rule gives
		\begin{equation}\label{eq:trajectory-chain}
			\begin{aligned}
				\mathbb E_X\KL(P_{g_\star}^M(\cdot\mid X)\|P_{\widehat g}^M(\cdot\mid X))
				&=\sum_{t=1}^M\mathbb E\KL(p_{g_\star}(XZ_{<t})\|\widehat p(XZ_{<t}))\\
				&=M\mathbb E_{S\sim\mu}\KL(p_{g_\star}(S)\|\widehat p(S))\le\varepsilon.
			\end{aligned}
		\end{equation}
		All expectations are under $X\sim P$ and the target trajectory.
		Marginalization to $Z_M$ proves~\eqref{eq:base-cot}. 
	\end{proof}
	
	\subsection{From chain-of-thought to end-to-end supervision}
	\begin{theorem}\label{thm:main_2}
		For every $\cF$, $M\ge1$, and $0<\varepsilon\le1$,
		\begin{equation}\label{eq:e2e-cot}
			m_{\etoe}^{\cF,M}(\varepsilon)
			\le\frac{40M}{\varepsilon}
			\bigl(1+m_{\ct}^{\cF,M}(\varepsilon/24)\bigr).
		\end{equation}
	\end{theorem}
	\begin{proof}\linespread{0.94}\selectfont
		Set $n=m_{\ct}^{\cF,M}(\varepsilon/24)$; there is nothing to prove
		if $n=\infty$. Fix an $n$-sample CoT learner $A$ attaining the stated guarantee.
		Draw three independent groups of $n$ prompts from $P$ and
		independent random seeds for $A$, one per group.
		Given its prompts $(X_i)_{i=1}^n$, run $A$ on
		$((X_i,z_i))_{i=1}^n$ for every $(z_1,\ldots,z_n)\in(\{0,1\}^M)^n$,
		using the fixed seed within each group. Denote the resulting
		$K=3\cdot2^{nM}$ predictors by $q_1,\ldots,q_K$.
		The total-learner convention defines $A$ on every assignment.
		
		Augment each group by independent target trajectories conditional
		on its prompts. Its enumeration contains $A$'s corresponding output,
		of risk at most $\varepsilon/24$ with probability at least $2/3$.
		Independence of groups and seeds gives
		\[
		\Pr\!\left\{\min_{j\in[K]}\mathbb E_{X\sim P}
		\KL(q_{g_\star}^{\etoe-M}(X)\|q_j(X))>\varepsilon/24\right\}
		\le3^{-3}.
		\]
		Let $G$ be the complementary event in the displayed inequality.
		Apply Lemma~\ref{lem:aggregation} to $T=\lceil24\log(K+1)/\varepsilon\rceil$
		fresh end-to-end samples, independently of the list.
		Conditional on $G$ and the list, the expected risk is at most
		$\varepsilon/24+\log(K+1)/T\le\varepsilon/12$.
		Markov bounds the conditional failure probability by $1/12$.
		The failure probability is at most $3^{-3}+1/12<1/3$.
		Since $K+1\le4\cdot2^{nM}$, we have
		$3n+T\le40M(n+1)/\varepsilon$, proving~\eqref{eq:e2e-cot}, also for $n=0$.
	\end{proof}

	Combining the two reductions also compares base and end-to-end learning.
	\begin{corollary}\label{cor:base-to-e2e}
		For every $\cF$, $M\ge1$, and $0<\varepsilon\le1$,
		\[
		m_{\etoe}^{\cF,M}(\varepsilon)
		\le\frac{40M}{\varepsilon}
		\bigl(1+m_{\base}^{\cF}(\varepsilon/(24M))\bigr).
		\]
		Moreover, for every $0<\delta<1$, $m_{\ct}^{\cF,M}(\varepsilon,\delta)
		\le m_{\etoe}^{\cF,M}(\varepsilon,\delta)$.
	\end{corollary}
	\begin{proof}\linespread{0.94}\selectfont
		The first assertion follows from Theorems~\ref{thm:base-to-cot-full}
		and~\ref{thm:main_2}. For the second, discard
		$Z_{i,1},\ldots,Z_{i,M-1}$ from each $\ct$ sample and apply
		an $\etoe$ learner to $(X_i,Z_{i,M})_{i=1}^n$.
	\end{proof}
	
Repeating the construction on independent groups of samples
and using fresh validation samples gives the following
arbitrary confidence bound.
	\begin{theorem}\label{thm:kl-comparison-confidence}
		For every $\cF$, $M\ge1$, $0<\varepsilon\le1$, and $0<\delta<1$,
		let $n=m_{\ct}^{\cF,M}(\varepsilon/24,1/3)$ and
		$k=\lceil\log(2/\delta)/\log3\rceil$. Then
		\[
		m_{\etoe}^{\cF,M}(\varepsilon,\delta)
		\le kn+\left\lceil\frac{20\log(256/\varepsilon)}{\varepsilon}
		\left(nM\log2+\log\frac{2k}{\delta}\right)\right\rceil.
		\]
		The same holds with $n=m_{\base}^{\cF}(\varepsilon/(24M),1/3)$.
	\end{theorem}
	\begin{proof}\linespread{0.94}\selectfont
		For $n=\infty$ the bound is immediate. Otherwise use the
		enumeration in Theorem~\ref{thm:main_2} to $k$ independent
		groups, obtaining $K=k2^{nM}$ predictors. Their minimum risk is
		at most $\varepsilon/24$ with probability at least
		$1-3^{-k}\ge1-\delta/2$. Conditional on this event,
		Lemma~\ref{lem:kl-selection} with confidence $1-\delta/2$
		uses at most $\lceil(20B/\varepsilon)\log(2K/\delta)\rceil$
		fresh samples, where $B=\log(256/\varepsilon)$.
		A union bound proves the displayed estimate.
		Theorem~\ref{thm:base-to-cot-full} supplies the base version.
	\end{proof}
	
	\section{Classes for which the Upper Bounds are Tight}\label{sec:sharp}
We first show how to construct generators whose final-token probabilities differ only at a prescribed finite set of prompts.
	\begin{lemma}\label{lem:roots}
		For integers $M\ge2$ and $K\ge1$, there are distinct strings
		$r_1,\ldots,r_K$ such that arbitrary one-step laws may be prescribed
		at $r_jz$ for $z\in\Sigma^\star$, $|z|<M$, while the final law at every prompt
		other than the $r_j$ is independent of these prescriptions.
	\end{lemma}
\begin{proof}\linespread{0.94}\selectfont
    Choose distinct equal-length strings $b_1,\ldots,b_K$ and set
    $r_j=b_j1$. Since the roots have equal length, their descendant
    sets $r_j\Sigma^\star$ are disjoint. Prescribe the given laws at
    $r_jz$ for $|z|<M$, set $p_g(b_j)=0$, and set $p_g(s)=1/2$
    at every other state.
    Fix a prompt $x$ distinct from all roots. If $x$ is outside
    every descendant set, its trajectory cannot enter one:
    the first entry would require generating $1$ at some $b_j$,
    where this has probability zero. Every state reached from $x$ therefore has a fixed
    transition law, independent of the prescriptions.
    Otherwise, $x=r_jz$ for some $j$ and $|z|\ge1$.
    Immediately before generating the final bit, the state is
    $r_jzZ_{<M}$, whose depth relative to $r_j$ is
    $|z|+M-1\ge M$. Hence $p_g(r_jzZ_{<M})=1/2$ for every history,
    independently of the prescriptions.
    In both cases, the final-bit distribution is independent
    of the prescriptions, as required.
	\end{proof}
	
	The lower bounds use the following fact. 
    
	\begin{lemma}\label{lem:cube-testing}
		Let $D\ge1$ and let $\omega^{(j)}$ denote $\omega\in\{-1,1\}^D$
		with its $j$th sign flipped. Suppose the law $\mathbb P_\omega$ of the training
		sample, indexed by $\omega\in\{-1,1\}^D$, satisfies
		$\KL(\mathbb P_\omega\|\mathbb P_{\omega^{(j)}})\le1/50$
		for every $\omega\in\{-1,1\}^D$ and $j\in[D]$. Then any estimator
		$\widehat\omega\in\{-1,1\}^D$ makes at least $D/16$
		coordinate errors with probability greater than $1/3$ for some $\omega$.
	\end{lemma}
	\begin{proof}\linespread{0.94}\selectfont
		Draw $\omega$ uniformly, the data under $\mathbb P_\omega$,
		and the learner's independent random bits; put $H=\sum_{j=1}^D\mathbf1\{\widehat\omega_j\ne\omega_j\}$.
		Pinsker gives $\operatorname{TV}(\mathbb P_\omega,
		\mathbb P_{\omega^{(j)}})\le1/10$. Pairing $\omega$ with
		$\omega^{(j)}$ in the binary test gives
		$\Pr\{\widehat\omega_j\ne\omega_j\}\ge9/20$.
		Thus $\mathbb EH\ge9D/20$ and
		\[
		\Pr\{H\ge D/16\}\ge
		\frac{\mathbb EH-D/16}{D-D/16}\ge\frac{31}{75}>\frac13.
		\]
		These probabilities include the data and learner randomness; averaging over
		$\omega$ gives a fixed sign vector with the claimed failure probability.
	\end{proof}
	If a coordinate has possible means $u-v$ and $u+v$, thresholding
	a probability estimate at $u$ gives an estimated sign. An incorrect
	sign costs KL at least $2v^2$ by~\eqref{eq:bernoulli-bounds}.
	
	\subsection{The accuracy \texorpdfstring{$\varepsilon/M$}{epsilon/M} in Theorem~\ref{thm:base-to-cot-full}}
	\begin{theorem}\label{thm:base-scale-lower}
		There are universal $c,C,\varepsilon_0>0$ such that, for every
		$K\ge1$, $M\ge2$, and $0<\varepsilon\le\varepsilon_0$, a finite class
		satisfies
		\[
		m_{\base}^{\cF}(C\varepsilon/M)=0,
		\qquad m_{\ct}^{\cF,M}(\varepsilon)\ge cK/\varepsilon.
		\]
		Consequently, for fixed $0<\beta<1$ the base sample complexity is
		also zero at accuracy $\varepsilon^{1-\beta}/M$ for sufficiently small
		$\varepsilon$, and at $\varepsilon/M^{1-\beta}$ for sufficiently large $M$.
	\end{theorem}
	\begin{proof}\linespread{0.94}\selectfont
		Take $K$ roots $r_j$ isolated as in Lemma~\ref{lem:roots}: only the roots
		can have parameter-dependent final-token laws. Index the generators by
		$\omega\in\{-1,1\}^K$. Set $a=32$, $p_0=1/M$, and $h=a\sqrt\varepsilon/M$, and assume
		$a\sqrt\varepsilon\le1/4$. For $|z|<M$, prescribe
		\begin{equation}\label{eq:hazard-law}
			p_{g_\omega}(r_jz)=
			\begin{cases}p_0+\omega_jh,&z=0^{|z|},\\1,&\text{otherwise},\end{cases}
			\qquad Q_M(p):=1-(1-p)^M.
		\end{equation}
		The final bit is $1$ exactly when some step emits $1$, so
		\eqref{eq:hazard-law} gives
		$q_{g_\omega}^{\etoe-M}(r_j)=Q_M(p_0+\omega_jh)$.
		The zero-sample base predictor outputs $p_0$ at the
		sign-dependent states and the prescribed transition elsewhere. The upper bound in
		\eqref{eq:bernoulli-bounds} gives pointwise KL at most
		$2Mh^2=2a^2\varepsilon/M$. For $a\sqrt\varepsilon\le1/4$, $p_0\pm h\in[3/(4M),5/(4M)]$.
		On this interval,
		$Q'_M(p)=M(1-p)^{M-1}\ge Me^{-5/3}>M/8$; here
		$\log(1-u)\ge-u/(1-u)$ for $0\le u<1$.
		Thus $Q_M(p_0+h)-Q_M(p_0-h)\ge a\sqrt\varepsilon/4$.
		Let $P$ be uniform on the $K$ roots. Flipping $\omega_j$
		changes the trajectory law only at $r_j$. At each visited state before absorption, the KL of the two
		Bernoulli laws is at most $C_0Mh^2$. The chain rule bounds the
		$n$-sample divergence between adjacent sign vectors by
		$C_0nM^2h^2/K=C_0na^2\varepsilon/K$ (enlarging $C_0$ if needed).
		For $n\le cK/\varepsilon$ and sufficiently small universal $c$,
		Lemma~\ref{lem:cube-testing} gives at least $K/16$
		incorrect signs with probability greater than $1/3$.
		Decode each sign by thresholding the predicted final mean at
		$[Q_M(p_0+h)+Q_M(p_0-h)]/2$. Every error costs KL at least
		$a^2\varepsilon/32$ by~\eqref{eq:bernoulli-bounds}. The risk is at least
		$a^2\varepsilon/512=2\varepsilon$.
		Finally take $C=2a^2$. Both consequences follow by comparing their
		specified accuracies with $C\varepsilon/M$.
	\end{proof}
	
	Adding roots that affect only the base problem makes the two sample costs match.
	\begin{theorem}\label{thm:base-match}
		For a universal $\kappa>0$, every real $L\ge1$, integer $M\ge2$,
		and all sufficiently small $\varepsilon>0$, some finite class satisfies
		\[
		m_{\ct}^{\cF,M}(\varepsilon)=\Theta(LM/\varepsilon),\qquad
		m_{\base}^{\cF}(\kappa\varepsilon/M)=\Theta(LM/\varepsilon).
		\]
	\end{theorem}
\begin{proof}\linespread{0.94}\selectfont
    Set $a=32$, $K=\lceil LM\rceil$, $E=\lceil L\rceil$,
    $\kappa=24a^2$, and fix $b>0$ with $b^2>16\kappa$.
    Take $K+E$ isolated roots $r_1,\ldots,r_K,v_1,\ldots,v_E$
    by Lemma~\ref{lem:roots}. Define
    $\cF=\{g_{\omega,\tau}:\omega\in\{-1,1\}^K,
    \tau\in\{-1,1\}^E\}$ with $p_0=1/M$, $h=a\sqrt\varepsilon/M$,
    and $p_{g_{\omega,\tau}}(r_jz)$ given by the right-hand
    side of~\eqref{eq:hazard-law} for every $|z|<M$.
    At each $v_\ell$, set
    $p_{g_{\omega,\tau}}(v_\ell)=
    1/2+\tau_\ell b\sqrt{\varepsilon/M}$ and make its
    strict descendants fair (i.e., yielding $1$ with probability $1/2$). Complete all other transitions
    by Lemma~\ref{lem:roots}. Take $\varepsilon$ sufficiently small
    that all probabilities are valid.

    We first bound the $\ct$ sample complexity.
    Fixing $\tau$ and taking $P$ uniform on the $r_j$,
    the proof of Theorem~\ref{thm:base-scale-lower}
    gives $m_{\ct}^{\cF,M}(\varepsilon)\ge cK/\varepsilon$.
    Since $M\ge2$, the final means at the $v_\ell$ are $1/2$,
    and all non-root final means are fixed. Thus there are at
    most $2^K$ distinct final-token predictors.
    Lemma~\ref{lem:aggregation} and Markov's inequality give
    $m_{\ct}^{\cF,M}(\varepsilon)
    \le\lceil3\log(2^K+1)/\varepsilon\rceil=O(K/\varepsilon)$.

    For each $\tau\in\{-1,1\}^E$, define a base predictor
    using $p_0$ at all $r_j0^t$, $0\le t<M$, and the given
    $\tau$-dependent laws at $v_\ell$; all other laws are fixed.
    For each target, one predictor has pointwise KL at most
    $2Mh^2=2a^2\varepsilon/M$.
    Lemma~\ref{lem:aggregation}, with
    $n=\lceil LM/\varepsilon\rceil$, gives, for each target,
    expected KL at most
    $2a^2\varepsilon/M+\log(2^E+1)/n
    \le\kappa\varepsilon/(3M)$, since $E\le2L$.
    Markov's inequality gives
    $m_{\base}^{\cF}(\kappa\varepsilon/M)=O(LM/\varepsilon)$.

    For the base lower bound, fix $\omega$ and let $P$ be
    uniform on the $v_\ell$. Put $t=b\sqrt{\varepsilon/M}$,
    and let $\mathbb P_\tau$ be the law of $n$ base samples.
    For $\tau^{(\ell)}$ obtained by flipping $\tau_\ell$,
    $\KL(\mathbb P_\tau\|\mathbb P_{\tau^{(\ell)}})
    =(n/E)\KL(1/2+t\|1/2-t)
    \le16nb^2\varepsilon/(EM)\le1/50$
    whenever $n\le LM/(800b^2\varepsilon)$.
    Lemma~\ref{lem:cube-testing}, applied by thresholding
    $\widehat p(v_\ell)$ at $1/2$, gives some $\tau$
    for which at least $E/16$ signs are incorrect with
    probability greater than $1/3$.
    On this event, Pinsker's inequality gives KL risk at least
    $2t^2/16=b^2\varepsilon/(8M)>
    \kappa\varepsilon/M$.
    Hence $m_{\base}^{\cF}(\kappa\varepsilon/M)
    =\Omega(LM/\varepsilon)$, completing the proof.
	\end{proof}
	
	\subsection{The factor \texorpdfstring{$M/\varepsilon$}{M/epsilon} in Theorem~\ref{thm:main_2}}
	\begin{theorem}\label{thm:overhead-lower}
		For fixed $0<\rho\le1$, integers $N\ge1$, $M\ge2$, and all sufficiently
		small $\varepsilon>0$, there exists a finite class with
		\[
		m_{\ct}^{\cF,M}(\rho\varepsilon)=\Theta_\rho(N),\qquad
		m_{\etoe}^{\cF,M}(\varepsilon)=\Theta(NM/\varepsilon).
		\]
	\end{theorem}
	\begin{proof}\linespread{0.94}\selectfont
		Put $T=M-1$ and use the isolated roots $r_{j,i}$,
		$(j,i)\in[N]\times[T]$. Let $\cF=\{g_\omega:\omega\in\{-1,1\}^{N\times T}\}$,
		with $h=8\sqrt\varepsilon\le1/8$ and transitions
		$p_{g_\omega}(r_{j,i}z)=1/2+\omega_{j,|z|+1}/4$
		for $|z|<T$, and $p_{g_\omega}(r_{j,i}z)=
		1/2+2h(2z_i-1)$ for $z\in\{0,1\}^T$.
		The final mean at $r_{j,i}$ is exactly $1/2+h\omega_{j,i}$.
		
		Let $N_j$ count samples with prompts in row $j$. For each
		$i\in[T]$, threshold their $i$th-bit empirical mean at $1/2$
		to obtain $\widehat\omega_{j,i}$ (set it to $1$ if $N_j=0$).
		Predict $1/2+h\widehat\omega_{j,i}$ at $r_{j,i}$.
		Hoeffding gives
		$\Pr\{\widehat\omega_{j,i}\ne\omega_{j,i}\mid N_j\}\le e^{-N_j/8}$. For any prompt law $P$, let
		$p_j=P(\{r_{j,i}:i\in[T]\})$; then $N_j\sim\operatorname{Bin}(n,p_j)$.
		By~\eqref{eq:bernoulli-bounds}, the expected risk is at most
		\begin{equation}\label{eq:row-risk}
			Ch^2\sum_{j=1}^N p_j(1-p_j+p_je^{-1/8})^n
			\le C\varepsilon\sum_{j=1}^Np_je^{-c'np_j}
			\le C'\varepsilon N/n,
		\end{equation}
		where $c'=1-e^{-1/8}$, $C'$ is universal, and the last step uses
		$u e^{-c'nu}\le1/(ec'n)$ for $u\ge0$ and $n\ge1$.
		The non-root means are known by Lemma~\ref{lem:roots}.
		Markov's inequality applied to~\eqref{eq:row-risk} gives the CoT upper
		bound for $n\ge C_\rho N$.
		For the lower bound, fix any learner's output $\widehat q$
		and decode each sign by thresholding $\widehat q(r_{j,i})$
		at $1/2$. Draw $\omega$ uniformly and let $P$ be uniform
		on the $NT$ roots. If $n\le N/8$, at least
		$7N/8$ rows are unseen, so their signs remain independent
		fair bits. For $H$ incorrectly decoded signs, conditional on the
		observed data, $\mathbb EH\ge7NT/16$; hence
		$\Pr\{H\ge NT/16\}\ge(7/16-1/16)/(1-1/16)>1/3$. Every such sign incurs KL at least
		$2h^2=128\varepsilon>16\rho\varepsilon$, proving
		$m_{\ct}(\rho\varepsilon)=\Omega(N)$.
		
		Under $\etoe$ supervision, the observation at $r_{j,i}$
		is $\Ber(1/2+h\omega_{j,i})$. For adjacent $\omega,\omega'$
		with one sign flipped, the $n$-sample KL is at most
		$Cn h^2/(NT)=O(n\varepsilon/(NT))$. Lemma~\ref{lem:cube-testing} yields
		$n=\Omega(NT/\varepsilon)$. Conversely, the endpoint class has
		$2^{NT}$ members, and Proposition~\ref{prop:trajectory} gives
		$O(NT/\varepsilon)$ samples. Since $M/2\le T<M$, the result follows.
	\end{proof}
	
	\section{Simultaneous exact-accuracy separation}\label{sec:same_scale}
	\begin{theorem}\label{thm:simultaneous-exact-scale-taxonomy}
		There are universal $M_0$ and $\varepsilon_0>0$ such that, for every
		$M\ge M_0$, $0<\varepsilon\le\varepsilon_0$, and $A\ge1$, a finite
		class $\cF$ satisfies $0<m_{\base}^{\cF}(\varepsilon)<\infty$ and
		\[
		\frac{m_{\ct}^{\cF,M}(\varepsilon)}{m_{\base}^{\cF}(\varepsilon)}
		>\frac{AM}{\varepsilon},\qquad
		\frac{m_{\etoe}^{\cF,M}(\varepsilon)}{m_{\ct}^{\cF,M}(\varepsilon)}
		>\frac{AM}{\varepsilon}.
		\]
	\end{theorem}
	We build two finite families on disjoint sets of isolated roots.
	The first separates base from CoT; the second separates CoT from end-to-end.
	
	\paragraph{Majority amplification.}
	For odd $J$, set $F_J(t)=\Pr\{\operatorname{Bin}(J,1/2+t)>J/2\}$.
	Since
	$F'_J(0)=J\binom{J-1}{(J-1)/2}2^{-(J-1)}\sim\sqrt{2J/\pi}$,
	fix an odd $J$ with $F'_J(0)\ge200$. For $\varepsilon>0$ small enough, let $h_\varepsilon\in(0,1/4)$
	be the unique solution of $\KL(1/2+h_\varepsilon\|1/2)=\varepsilon$.
	For all sufficiently small $\varepsilon$ there are positive
	$t_G,t_H$ such that
	\begin{equation}\label{eq:majority-biases}
		F_J(t_G)=1/2+8\sqrt\varepsilon,
		\qquad F_J(t_H)=1/2+h_\varepsilon,
		\qquad \KL(1/2\pm t_G\|1/2),
		\ \KL(1/2\pm t_H\|1/2)\le\varepsilon/100.
	\end{equation}
	Indeed $F_J$ is strictly increasing, $F_J(0)=1/2$, and
	$\KL(1/2+t\|1/2)=2t^2+O(t^4)$. Hence
	\[
	\frac{h_\varepsilon}{\sqrt\varepsilon}\longrightarrow\frac1{\sqrt2},
	\qquad
	\frac{t_G}{\sqrt\varepsilon}\longrightarrow\frac8{F'_J(0)},
	\qquad
	\frac{t_H}{\sqrt\varepsilon}\longrightarrow\frac1{\sqrt2F'_J(0)}.
	\]
	These limits and symmetry give~\eqref{eq:majority-biases}.
	Fix $M\ge M_0:=J+3$.
	
	\paragraph{The first family $G_N$.}
	For $N\ge1$ use $N$ isolated roots $r_j$ and one base-only root $v$.
	Define $G_N$ to have one generator $g_{\omega,\tau}$ for each
	$\omega\in\{-1,1\}^N$ and $\tau\in\{-1,1\}$, with transitions
	specified below and elsewhere by Lemma~\ref{lem:roots}.
	At $v$, prescribe $p_{g_{\omega,\tau}}(v)=1/2+4\tau\sqrt\varepsilon$ and fair
	transitions at its strict descendants. At $r_j$, for every $|z|<M$, set
	\[
	p_g(r_jz)=\begin{cases}
		1/2+t_G\omega_j,&|z|<J,\\
		\mathbf1\{\sum_{t=1}^Jz_t>J/2\},&|z|=M-1,\\
		1/2,&\text{otherwise}.
	\end{cases}
	\]
	Thus $q_g^{\etoe-M}(r_j)=1/2+8\omega_j\sqrt\varepsilon$
	and $q_g^{\etoe-M}(v)=1/2$. We prove
	\[
	m_{\base}^{G_N}(\varepsilon)=\Theta(1/\varepsilon),
	\qquad m_{\ct}^{G_N,M}(\varepsilon)=\Theta(N/\varepsilon).
	\]
	For the base upper bound, use two one-step candidates indexed by
	$\tau$, setting all $J$ weak coins to $1/2$.
	By~\eqref{eq:majority-biases} one candidate has pointwise KL at most
	$\varepsilon/100$; Lemma~\ref{lem:aggregation} gives
	$O(1/\varepsilon)$ samples. With $P$ concentrated on $v$, the adjacent $n$-sample KL is
	$n\KL(1/2+4\sqrt\varepsilon\|1/2-4\sqrt\varepsilon)=O(n\varepsilon)$.
	A wrong sign, decoded at $1/2$, costs KL at least $32\varepsilon$.
	Pinsker's inequality gives the base lower bound $\Omega(1/\varepsilon)$.
	For CoT, the distinct endpoint functions number $2^N$, giving an
	$O(N/\varepsilon)$ bound by Proposition~\ref{prop:trajectory}.
	With $P$ uniform on $r_j$, adjacent $n$-sample laws have KL at most
	$(nJ/N)\KL(1/2+t_G\|1/2-t_G)\le Cn\varepsilon/N$.
	For $n\le cN/\varepsilon$, Lemma~\ref{lem:cube-testing}
	gives at least $N/16$ sign errors, each costing $128\varepsilon$;
	thus $m_{\ct}^{G_N,M}(\varepsilon)=\Omega(N/\varepsilon)$.

	\paragraph{The second family $H_D$.}
	Use roots $x_0,\ldots,x_D$, with $x_0$ called the anchor.
	The class $H_D$ contains the inactive generator $h_0$ and active generators
	$h_{b,s}$ for $b\in\{0,1\}$, $s\in\{-1,1\}^D$. For $0\le j\le D$, set
	$p_{h_0}(x_j)=0$ and $p_{h_0}(x_jz)=1/2$ for $0<|z|<M$.
	For every $0\le j\le D$ and $|z|<M$, define
	\[
	p_{h_{b,s}}(x_jz)=\begin{cases}
		1,&|z|=0,\\
		b,&|z|=1\text{ or }(j=0,\ |z|=M-1),\\
		1/2+t_Hs_j,&j\ge1,\ 2\le|z|<J+2,\\
		\mathbf1\{\sum_{t=3}^{J+2}z_t>J/2\},&j\ge1,\ |z|=M-1,\\
		1/2,&\text{otherwise}.
	\end{cases}
	\]
	All other transitions are fixed by Lemma~\ref{lem:roots}.
	Thus, for an active model,
	\[
	q_{h_{b,s}}^{\etoe-M}(x_0)=b,\qquad
	q_{h_{b,s}}^{\etoe-M}(x_j)=1/2+h_\varepsilon s_j\quad(j\in[D]).
	\]
	In the inactive model every root has final mean $1/2$.
	Three base candidates (inactive, and active for each $b$) with the
	weak coins replaced by fair coins have approximation KL at most
	$\varepsilon/100$ by~\eqref{eq:majority-biases}; thus
	$m_{\base}^{H_D}(\varepsilon)=O(1/\varepsilon)$ independently of $D$.
	
	For a $\ct$ sample at an $H_D$ root, $Z_1=1$ exactly when
	the target is active, and then $Z_2=b$. Once $Z_1=1$ is observed,
	predict $b$ at the anchor and $1/2$ at each hidden root;
	the resulting KL risk is at most $\varepsilon$.
	Until then predict the inactive endpoint law. Set $p=P(\{x_0,\ldots,x_D\})$. Under an active target, the latter
	prediction has risk at most $p\log2$. If $p\log2>\varepsilon$, then
	$n\ge(\log3)(\log2)/\varepsilon$ gives
	$\Pr\{\text{no root sampled}\}=(1-p)^n\le e^{-np}<1/3$;
	if $p\log2\le\varepsilon$, the prediction already suffices. Therefore
	\[
	m_{\ct}^{H_D,M}(\varepsilon)=O(1/\varepsilon)
	\qquad\text{uniformly in }D.
	\]

	\begin{lemma}\label{lem:plateau}
		Fix $J$ as in~\eqref{eq:majority-biases}, $M\ge J+3$,
		and sufficiently small $\varepsilon>0$. For every integer $n_0\ge0$,
		some $D\ge1$ satisfies
		$m_{\etoe}^{H_D,M}(\varepsilon)>n_0$.
	\end{lemma}
	\begin{proof}\linespread{0.94}\selectfont\setlength{\jot}{0pt}
		Fix an arbitrary $n_0$-sample learner with output $\widehat q$. Set $\eta=1/[100(n_0+1)]$. Choose $D$ as in~\eqref{eq:plateau-size} below
		(independently of the learner), and put
		$P(x_0)=\eta$, and $P(x_j)=(1-\eta)/D$ for $j\in[D]$.
		Draw $b,s_1,\ldots,s_D$ independently and uniformly.
		Let $E$ be the event that no training prompt equals $x_0$; then
		$\Pr(E)\ge1-n_0\eta\ge0.99$.
		Condition on $E$, the data, and the learner's seed.
		Let $U\subseteq[D]$ index the unseen roots, so $|U|\ge D-n_0$.
		The variables $b,(s_j)_{j\in U}$ remain independent and uniform.
		
		Let $u_j$ be the projection of $\widehat q(x_j)$ onto
		$[1/2-h_\varepsilon,1/2+h_\varepsilon]$; this cannot increase KL.
		Set $L=(1-\eta)D^{-1}\sum_{j\in U}\KL(1/2+h_\varepsilon s_j\|u_j)$.
		The original risk is at least $L$ plus its anchor contribution.
		All expectations below are conditional on the fixed data and seed.
		For $j\in U$,
		\begin{align}
			\KL(1/2+h_\varepsilon s_j\|u_j)
			&=\varepsilon+\KL(1/2\|u_j)
			+h_\varepsilon s_j\log\frac{1-u_j}{u_j},\label{eq:plateau-loss}\\
			0\le\KL(1/2+h_\varepsilon s_j\|u_j)
			&\le\frac{4h_\varepsilon^2}{1/4-h_\varepsilon^2}\le12\varepsilon.
			\label{eq:plateau-bound}
		\end{align}
		Equation~\eqref{eq:plateau-bound} uses~\eqref{eq:bernoulli-bounds}
		and $2h_\varepsilon^2\le\varepsilon$.
		By~\eqref{eq:plateau-loss}, independence, and~\eqref{eq:plateau-bound},
		\begin{align}
			\mathbb EL&\ge(1-\eta)(1-n_0/D)\varepsilon,\label{eq:plateau-mean}\\
			\operatorname{Var}(L)&\le144\varepsilon^2/D.\label{eq:plateau-var}
		\end{align}
		Take $\varepsilon<\log2/8$ and an integer
		\begin{equation}\label{eq:plateau-size}
			D\ge\max\left\{1,\frac{8n_0\varepsilon}{\eta\log2},
			\frac{10^6\varepsilon^2}{\eta^2(\log2)^2}\right\}.
		\end{equation}
		Equations~\eqref{eq:plateau-mean}--\eqref{eq:plateau-size} and Chebyshev give
		\[
		\Pr\{L<\varepsilon-\eta\log2/2\}
		\le\frac{144\varepsilon^2/D}{(\eta\log2/4)^2}<0.01.
		\]
		Conditional on $E$, the data, and the learner seed, $b$ and
		$(s_j)_{j\in U}$ are independent, and $u_0=\widehat q(x_0)$ satisfies
		$\Pr_b\{\KL(b\|u_0)\ge\log2\}\ge1/2$.
		Consequently, the original risk exceeds $\varepsilon$ with
		probability at least $\Pr(E)(0.99)/2\ge0.99^2/2>1/3$. Averaging over $(b,s)$ supplies a fixed
		failing target. This holds for every $n_0$-sample learner.
	\end{proof}
	
	\begin{proof}[Proof of Theorem~\ref{thm:simultaneous-exact-scale-taxonomy}]\linespread{0.94}\selectfont
		Place all roots of $G_N$ and $H_D$ using one application of
		Lemma~\ref{lem:roots}. Fix $g_0\in G_N$ and combine the constructions as
		\[
		\cF=\{g\oplus h_0:g\in G_N\}\cup\{g_0\oplus h:h\in H_D\}.
		\]
		The operation $\oplus$ uses the first generator on its root
		trees, the second on its disjoint root trees, and the common
		transitions of Lemma~\ref{lem:roots} elsewhere.
		Use four base candidates: two for $G_N\oplus h_0$, indexed by $\tau$,
		and two for $g_0\oplus H_D$, indexed by $b$. Replace probabilities
		$1/2\pm t_G$ and $1/2\pm t_H$ by $1/2$. Some candidate has pointwise KL at most $\varepsilon/100$.
		Lemma~\ref{lem:aggregation} and testing $\tau$ at $v$ give constants
		$c_0,C_0>0$ independent of $N,D$ such that
		\begin{equation}\label{eq:union-base}
			c_0/\varepsilon\le m_{\base}^{\cF}(\varepsilon)\le C_0/\varepsilon.
		\end{equation}
		
		We construct the CoT learner. Let $R_G=\{v,r_1,\ldots,r_N\}$ and
		$R_H=\{x_0,\ldots,x_D\}$. For each trajectory set $Y=Z_M\mathbf1\{X\in R_G\}$;
		aggregate the $2^N$ predictors equal to the $G_N$ final
		means on $R_G$ and zero elsewhere. With
		$n=\lceil CN/\varepsilon\rceil$ independent trajectories,
		Lemma~\ref{lem:aggregation} and Markov give an estimate $\widehat q_G$
		satisfying, with probability at least $5/6$,
		\begin{equation}\label{eq:union-G-risk}
			\mathbb E_X\!\left[\mathbf1\{X\in R_G\}
			\KL(q_\star(X)\|\widehat q_G(X))\right]\le\varepsilon/6.
		\end{equation}
		If an $R_H$ sample starts with $Z_1=1$, use the fixed
		$g_0$ means on $R_G$, its observed $b=Z_2$ at $x_0$,
		and $1/2$ at all $x_j$, $j\ge1$.
		This has risk at most $\varepsilon$. Otherwise use $\widehat q_G$
		on $R_G$ and the inactive law elsewhere. Non-root laws are always known.
		For an inactive target,~\eqref{eq:union-G-risk} suffices.
		For an active target, put $p=P(R_H)$, where every
		$R_H$ trajectory reveals activity and $b$:
		\begin{align*}
			p\log2\le\varepsilon/3&\implies
			\text{risk on event~\eqref{eq:union-G-risk}}\le\varepsilon/2,\\
			p\log2>\varepsilon/3&\implies
			\Pr\{\text{activity unseen}\}\le e^{-np}\le1/6.
		\end{align*}
		Thus a union bound, and the lower bound from $G_N$, give constants
		$c_1,C_1>0$ with
		\begin{equation}\label{eq:union-cot}
			c_1N/\varepsilon\le m_{\ct}^{\cF,M}(\varepsilon)\le C_1N/\varepsilon.
		\end{equation}
		Choose integers $N>AC_0M/(c_1\varepsilon)$ and
		$n_0>AC_1MN/\varepsilon^2$. Lemma~\ref{lem:plateau} supplies $D$
		with $m_{\etoe}^{\cF,M}(\varepsilon)>n_0$.
		Both ratios follow from~\eqref{eq:union-base} and~\eqref{eq:union-cot}.
		Finiteness follows from Proposition~\ref{prop:trajectory}.
	\end{proof}
	
	\section{Logistic autoregressive learning}\label{sec:logistic}
	For the class in~\eqref{eq:affine-logistic}, set
	$\theta=(\beta,w)\in\mathbb R^{d+1}$ and
	$\varphi_d(s)=(1,\tail_d(s))\in\{0,1\}^{d+1}$.
	Thus $p_{g_\theta}(s)=\sigma(\langle\theta,\varphi_d(s)\rangle)$;
	we call $\langle\theta,\varphi_d(s)\rangle$ the logit.
	Write $q_\theta^{\etoe-M}:=q_{g_\theta}^{\etoe-M}$.
	We prove the statistical bounds first, then the computational separation.
	
	\subsection{Full-trace learning and forward KL}
	For laws $Q,R$ on a finite or countable set, define normalized squared Hellinger
	distance by
	\[
	h^2(Q,R)=\frac12\sum_z(\sqrt{Q(z)}-\sqrt{R(z)})^2.
	\]
	For $u,v\in[0,1]$, write $h^2(u,v)=h^2(\Ber(u),\Ber(v))$.
	Put $S_a(v)=(1-a)v+a/2$ for $0<a<1$.
	
	\begin{lemma}\label{lem:smooth}
		For $u,v\in[0,1]$ and $0<a<1$,
		\begin{equation}\label{eq:smooth-hellinger-kl}
			\KL(u\|S_a(v))
			\le2(2+\log(2/a))\bigl(h^2(u,v)+a\bigr).
		\end{equation}
		Moreover, if $v,v'\in[a/2,1-a/2]$, then
		$|\KL(u\|v)-\KL(u\|v')|\le(2/a)|v-v'|$.
	\end{lemma}
	\begin{proof}\linespread{0.94}\selectfont
		Let $F(v)=2v^2\log v-v^2+1$. Then
		$F(1)=F'(1)=0$ and
		$F''(v)=4(1+\log v)\le4+2\log B$ for $0<v\le\sqrt B$, $B\ge1$.
		Taylor's formula and continuity at zero, with $t=v^2$, give
		\[
		t\log t-t+1\le(2+\log B)(\sqrt t-1)^2.
		\]
		Set $B=2/a$. Each likelihood ratio of $\Ber(u)$ against
		$\Ber(S_a(v))$ is at most $B$. Apply the preceding inequality
		to both likelihood ratios, weighting by their predicted masses;
		the linear terms sum to zero. Hence the KL is at most
		$2(2+\log(2/a))h^2(u,S_a(v))$.
		By convexity in the second argument,
		$h^2(u,S_a(v))\le(1-a)h^2(u,v)+a h^2(u,1/2)
		\le h^2(u,v)+a$.
		The final assertion follows from
		$|\partial_v\KL(u\|v)|=|v-u|/[v(1-v)]\le2/a$ on $[a/2,1-a/2]$.
	\end{proof}
	Projection to the final bit cannot increase Hellinger distance. Thus, if an estimated generator $\widehat g$
	satisfies $\mathbb E_Xh^2(P_{g_\star}^M(\cdot\mid X),
	P_{\widehat g}^M(\cdot\mid X))\le r$, then
	$S_a(q_{\widehat g}^{\etoe-M})$ has forward-KL risk at most
	$2(2+\log(2/a))(r+a)$.
	
	\paragraph{External full-trace theorem.}
	We use the following result of Dewasurendra
	\cite[Theorems 1--2]{dewasurendra2026fulltraces}.
	For $d,M\ge1$, $p=d+1$, $g_\star\in\cF_\sigma(d)$, any prompt
	law $P$, $0<r<1/2$, and $0<\delta<1/2$, an estimator attains
	\begin{equation}\label{eq:cited-trace-rate}
		\mathbb E_X h^2(P_{g_\star}^M(\cdot\mid X),
		P_{\widehat g}^M(\cdot\mid X))\le r
		\quad\text{with probability at least }1-\delta
	\end{equation}
	from at most
	\begin{equation}\label{eq:cited-trace-samples}
		\frac{C}{r}\left[p\log(CMp)\log(C/r)+\log\frac1\delta\right]
	\end{equation}
	full trajectories. The estimator need not be efficient. We verify its specialization below.
	\begin{proof}[Verification for the logistic class]\linespread{0.94}\selectfont
		Fix $\theta=(\theta_1,\ldots,\theta_p)\in\mathbb R^p$
		and set $u_j=e^{\theta_j}>0$. For $(x,z)\in\Sigma^\star\times\{0,1\}^M$,
		write $y_t=(y_{t1},\ldots,y_{tp})=\varphi_d(xz_{<t})$,
		$a_t(u)=\prod_{j=1}^p u_j^{y_{tj}}$, and
		\[
		f_\theta(z\mid x):=P_{g_\theta}^M(z\mid x)
		=\frac{\prod_{t=1}^M a_t(u)^{z_t}}
		{\prod_{t=1}^M(1+a_t(u))}.
		\]
		The numerator and denominator are polynomials in $u$ of
		degree at most $Mp$, with positive denominator. Thus for
		$b\in\mathbb R$, $f_\theta(z\mid x)>b$ is equivalent to a
		polynomial sign test of degree at most $Mp$ in $p$ variables.
		For $N\ge p$ fixed tests, the polynomial sign-pattern bound
		\cite{goldbergjerrum1995} gives at most $(CMpN/p)^p$ patterns.
		The subgraph dimension $V$ is the largest $N$ for which some $N$
		fixed triples $(x,z,b)$ realize all $2^N$ truth patterns as $\theta$
		varies. Hence $2^N\le(CMN)^p$ implies $V=O(p\log(CMp))$.
		
		For completeness, the statistical input is the VC-subgraph
		$\rho$-estimation theorem
		\cite[Proposition 7, Corollary 3, and Section 8]{baraudbirge2018rho}.
		For conditional densities of subgraph dimension $V$, it gives, in the
		realizable case and for $n\ge V$, a proper estimator satisfying
		\[
		\mathbb E_{X\sim P}h^2(P_{g_\star}^M(\cdot\mid X),
		P_{\widehat g}^M(\cdot\mid X))
		\le C\frac{V\log(en/V)+\log(1/\delta)}{n}
		\]
		with probability at least $1-\delta$.
		For our class, choose an additive $1$-minimizer over
		$\theta\in\mathbb Q^p$ of
		\[
		\sup_{\vartheta\in\mathbb Q^p}\sum_{i=1}^n
		\psi\!\left(\sqrt{\frac{f_\vartheta(Z_i\mid X_i)}
			{f_\theta(Z_i\mid X_i)}}\right),
		\qquad \psi(t)=\frac{t-1}{t+1}.
		\]
		Here $(X_i,Z_i)_{i=1}^n$ are the i.i.d. full trajectories.
		The joint density is $P(x)f_\theta(z\mid x)$; the factor $P(x)$
		cancels in each likelihood ratio. Continuity, density of
		$\mathbb Q^p$, and dominated convergence give zero approximation error.
		Taking $n$ as in~\eqref{eq:cited-trace-samples}, with large
		universal $C$, makes the displayed risk bound at most $r$.
	\end{proof}
	
	Smoothing the estimated final-token law converts this rate to forward KL.
	\begin{corollary}\label{cor:linear-cot-kl}
		For $d,M\ge1$, $0<\varepsilon<1$, and $0<\delta<1/2$, putting
		$p=d+1$, we have
		\[
		m_{\ct}^{\cF_\sigma(d),M}(\varepsilon,\delta)
		\le \frac{C\log(C/\varepsilon)}{\varepsilon}
		\left[p\log(CMp)\log(C/\varepsilon)+\log\frac1\delta\right].
		\]
	\end{corollary}
	\begin{proof}\linespread{0.94}\selectfont
		Use the cited estimator with
		$r=c\varepsilon/\log(C/\varepsilon)$, and output
		$S_r(q_{\widehat g}^{\etoe-M})$. By Lemma~\ref{lem:smooth}
		and Hellinger data processing, its KL risk is at most
		$4r(2+\log(2/r))\le\varepsilon$ when $c$ is small enough.
		Substitute $r$ into~\eqref{eq:cited-trace-samples}.
	\end{proof}
	
	\subsection{A finite cover for end-to-end learning}
	We approximate every logistic parameter vector on all Boolean features
	by a rational vector of bounded bit complexity.
	
	\begin{lemma}\label{lem:comparator}
		For integers $k,T\ge1$, $w\in\mathbb R^k$, and dyadic
		$0<\alpha\le1$, there is a rational $v\in\mathbb Q^k$, with
		$\|v\|_\infty\le R:=(k+1)!(T+1)$, such that for all
		$y\in\{0,1\}^k$,
		\begin{equation}\label{eq:compression}
			\begin{aligned}
				\langle w,y\rangle\ge T&\implies\langle v,y\rangle\ge T,\\
				\langle w,y\rangle\le-T&\implies\langle v,y\rangle\le-T,\\
				|\langle w,y\rangle|<T&\implies
				|\langle v-w,y\rangle|\le2\alpha.
			\end{aligned}
		\end{equation}
		Each coordinate of $v$ has
		$O(k\log(k+1)+\log(T+1)+\log(1/\alpha))$ bits.
	\end{lemma}
	\begin{proof}\linespread{0.94}\selectfont
		It suffices to treat $\alpha=2^{-b}$, $b\in\mathbb Z_{\ge0}$: replacing
		$\alpha$ by the largest power of two not exceeding it only strengthens the
		approximation and changes $\log(1/\alpha)$ by at most $\log2$.
		For each $y$ with $|\langle w,y\rangle|<T$, let
		$r_y=\alpha\lfloor\langle w,y\rangle/\alpha\rfloor$.
		Impose the first two constraints of~\eqref{eq:compression}
		and $|\langle v,y\rangle-r_y|\le\alpha$ for these $y$.
		Since $0\le\langle w,y\rangle-r_y<\alpha$, $w$ is feasible.
		Minimize $t$ subject to these constraints, $t\ge0$ and
		$-t\le v_i\le t$. Since $w$ is feasible, an optimum exists; at the optimal
		$t$ the feasible $v$-set is bounded and has a vertex. The $k+1$ coordinates $(v,t)$
		of this vertex solve $k+1$ independent tight constraints with
		coefficients in $\{-1,0,1\}$ and right sides of absolute value
		at most $T+1$. Cramer's rule gives
		$|v_i|\le(k+1)!(T+1)$, since a nonzero integer determinant has
		absolute value at least one. All right sides have dyadic
		denominators dividing that of $\alpha$, which gives the stated
		bit bound. Finally,
		$|\langle w,y\rangle-r_y|<\alpha$ proves the last line of
		\eqref{eq:compression}.
	\end{proof}
	
	We now discretize these bounded vectors to obtain a uniform one-step KL cover.
	\begin{lemma}\label{lem:kl-cover}
		For $k\ge1$ and $0<\eta<1/4$, there is an explicit finite
		$V_{k,\eta}\subset\mathbb Q^k$ with
		$\log|V_{k,\eta}|=O(k^2\log(k+1)+k\log(1/\eta))$
		such that every $w\in\mathbb R^k$ admits $v\in V_{k,\eta}$ with
		\begin{equation}\label{eq:local-cover}
			\sup_{y\in\{0,1\}^k}
			\KL\bigl(\sigma(\langle w,y\rangle)\|
			\sigma(\langle v,y\rangle)\bigr)\le\eta.
		\end{equation}
		For every $M\ge1$ and map $\psi:\Sigma^\star\to\{0,1\}^k$, let
		$q_w^{\etoe-M}$ be the final mean of the generator with one-step mean
		$\sigma(\langle w,\psi(s)\rangle)$. These final means satisfy
		$\sup_x\KL(q_w^{\etoe-M}(x)\|q_v^{\etoe-M}(x))\le M\eta$.
	\end{lemma}
	\begin{proof}\linespread{0.94}\selectfont
		Set
		$T=\lceil C[k\log(k+1)+\log(C/\eta)]\rceil$ and choose
		a power of two $\alpha\in(\eta/64,\eta/32]$.
		For fixed $w\in\mathbb R^k$, apply Lemma~\ref{lem:comparator}
		to obtain $v$ with $R=(k+1)!(T+1)$.
		For $y\in\{0,1\}^k$, put $u=\langle w,y\rangle$,
		$b=\langle v,y\rangle$, and $p=\sigma(u)$. The cross-entropy
		$L_p(b)=\log(1+e^b)-pb$ has derivative
		$\sigma(b)-p\in[-1,1]$. If $|u|<T$,
		$|u-b|\le2\alpha$ gives $\KL(p\|\sigma(b))\le2\alpha$.
		If $u\ge T$, then $T\le b\le kR$ and
		$1-p\le e^{-T}$. Since KL is at most cross-entropy,
		\[
		(1-p)b+\log(1+e^{-b})\le e^{-T}(kR+1)\le\eta/4.
		\]
		The negative-logit case is symmetric. The final inequality follows
		from the definition of $T$ for a sufficiently large universal $C$.
		Choose a power of two $\gamma\in(\eta/(32k),\eta/(16k)]$ and
		$V_{k,\eta}=(\gamma\mathbb Z\cap[-R,R])^k$.
		Rounding $v$ coordinatewise gives $v'\in V_{k,\eta}$;
		$|\langle v-v',y\rangle|\le k\gamma\le\eta/16$.
		The cross-entropy changes by at most $k\gamma\le\eta/16$. The full grid has at most
		$(1+2R/\gamma)^k$ points, so its logarithm has the claimed
		bound. These bounds imply~\eqref{eq:local-cover} with $v'$.
		The chain rule bounds trajectory KL by $M\eta$;
		projection to $Z_M$ proves the final claim.
	\end{proof}
	
	Aggregating the cover gives the end-to-end upper bound.
	\begin{theorem}\label{thm:tail-linear-sigmoid-e2e-upper}
		For $d,M\ge1$ and $0<\varepsilon<1$,
		\[
		m_{\etoe}^{\cF_\sigma(d),M}(\varepsilon)
		\le\frac C\varepsilon
		\left[d^2\log(d+1)+d\log\frac{CM}{\varepsilon}\right].
		\]
		For $0<\delta<1/2$, the bound at confidence $1-\delta$ is
		\[
		m_{\etoe}^{\cF_\sigma(d),M}(\varepsilon,\delta)
		\le\frac{C\log(C/\varepsilon)}{\varepsilon}
		\left[d^2\log(d+1)+d\log\frac{CM}{\varepsilon}
		+\log\frac1\delta\right].
		\]
	\end{theorem}
	\begin{proof}\linespread{0.94}\selectfont
		Apply Lemma~\ref{lem:kl-cover} with
		$k=d+1$, feature map $\varphi_d$, and
		$\eta=\varepsilon/(12M)$. The predictors
		$(q_v^{\etoe-M})_{v\in V_{k,\eta}}$ contain one with
		KL risk at most $\varepsilon/12$, uniformly in $P$.
		Lemma~\ref{lem:aggregation} with
		$T=\lceil8\log(|V_{k,\eta}|+1)/\varepsilon\rceil$
		gives expected risk at most $\varepsilon/12+\varepsilon/8<\varepsilon/3$;
		Markov's inequality gives the first bound. For confidence $1-\delta$, use
		$\eta=\varepsilon/(24M)$ and Lemma~\ref{lem:kl-selection}.
		Every endpoint on the grid is defined by a finite sum over its
		continuation tree. The resulting learner need not be efficient.
	\end{proof}
	
	\subsection{Quadratic end-to-end lower bound}
	We use the threshold lower bound of Doron-Arad, Mehalel, and Mossel~\cite[Theorem 1]{companioncapacity}:
	for a universal $c_0>0$, every sufficiently large
	$d$, and every $M\ge d^2$, there are at least $c_0d^2$ prompts
	in $\{0,1\}^d$ shattered by the deterministic $M$-step iterates of
	\[
	s\longmapsto\mathbf1\{\beta+\langle w,\tail_d(s)\rangle\ge0\},
	\qquad(\beta,w)\in\mathbb R^{d+1}.
	\]
	Here \emph{shattered} means that every binary labeling is realized
	by one shared parameter vector. The cited theorem uses threshold $1$,
	which is the subclass $\beta=-1$ above. We transfer its conclusion
	to finite logistic parameters.
	
	\begin{lemma}\label{lem:capacity-transfer}
		For $d,M\ge1$, if $x_0,\ldots,x_D\in\{0,1\}^d$ are
		shattered by the above $M$-step threshold iterates, then for every
		$\zeta>0$ and $b\in\{0,1\}^{D+1}$ there is a finite
		$\theta_b\in\mathbb R^{d+1}$ with
		$|q_{\theta_b}^{\etoe-M}(x_j)-b_j|\le\zeta$ for all $j$.
	\end{lemma}
	\begin{proof}\linespread{0.94}\selectfont
		Choose $\theta=(\beta,w)$ realizing $b$. Among the finitely
		many Boolean inputs, let $\gamma$ be the minimum magnitude
		of strictly negative logits (set $\gamma=1$ if none exist).
		Replacing $\beta$ by $\beta+\gamma/2$ preserves all
		threshold outputs and gives absolute logit at least
		$\mu:=\gamma/2$ on every Boolean input.
		Scale the perturbed vector by $A>0$. At each step, a
		logistic draw differs from its threshold bit with probability
		at most $e^{-A\mu}$. Using independent uniforms to couple corresponding steps,
		the final-token means differ by at most
		$Me^{-A\mu}\le\zeta$ when $A\mu\ge\log(M/\zeta)$.
	\end{proof}
	
	The transfer lets us apply a rare-prompt testing argument to logistic targets.
	\begin{theorem}\label{thm:logistic-e2e-lower}
		There are universal $c,\varepsilon_0>0$ and $d_0$ such that for
		$d\ge d_0$, $M\ge d^2$, and $0<\varepsilon\le\varepsilon_0$,
		\[
		m_{\etoe}^{\cF_\sigma(d),M}(\varepsilon)
		\ge c\,d^2/\varepsilon.
		\]
	\end{theorem}
	\begin{proof}\linespread{0.94}\selectfont
		Choose shattered prompts $x_0,\ldots,x_D$ with $D\ge c_0d^2/2$,
		increasing the dimension threshold if needed.
		Fix $b_0=0$ and draw $b_1,\ldots,b_D$ independently fair.
		Let $\alpha=256\varepsilon\le1/2$, and put
		$P(x_0)=1-\alpha$, $P(x_j)=\alpha/D$ for $j\ge1$.
		Fix any learner using $n\le n_0:=\lfloor D/(20\alpha)\rfloor$ samples.
		First consider the more informative experiment that reveals
		$b_j$ whenever the sampled prompt is $x_j$.
		At most $n\alpha\le D/20$ distinct rare prompts are seen
		in expectation. Each unseen $b_j$ remains independent and fair.
		Let $H=\sum_{j=1}^D\mathbf1\{\mathbf1\{\widehat q(x_j)\ge1/2\}\ne b_j\}$.
		Then
		\[
		\mathbb EH\ge19D/40,\qquad
		\Pr\{H\ge D/16\}\ge
		\frac{19/40-1/16}{1-1/16}=\frac{11}{25}.
		\]
		This includes the learner's randomness.
		
		For each labeling $b$, fix the logistic parameter $\theta_b$ from
		Lemma~\ref{lem:capacity-transfer} for
		$\zeta\le\min\{1/4,1/[100(n_0+1)]\}$.
		Couple each logistic label to $b_j$ at the same prompt;
		any disagreement over $n\le n_0$ samples has probability
		at most $n\zeta\le1/100$. Thus under the logistic experiments,
		$\Pr\{H\ge D/16\}>1/3$, averaged over labels and learner
		randomness. At each wrongly thresholded rare prompt,
		$|\widehat q(x_j)-q_{\theta_b}^{\etoe-M}(x_j)|\ge1/4$;
		by~\eqref{eq:bernoulli-bounds}, its KL is at least $1/8$.
		Consequently the risk on that event is at least
		$\alpha(D/16)/(8D)=2\varepsilon$.
		For every learner using $n\le n_0$ samples, some fixed finite
		logistic target therefore has failure probability greater than $1/3$.
	\end{proof}
	
	At the common horizon $M=d^2$, these bounds give the stated statistical separation.
	\begin{theorem}\label{thm:logistic-statistical-separation}
		For sufficiently large $d$ and sufficiently small $\varepsilon$,
		at $M=d^2$,
		\[
		m_{\ct}^{\cF_\sigma(d),d^2}(\varepsilon)
		\le\widetilde O(d/\varepsilon),\qquad
		m_{\etoe}^{\cF_\sigma(d),d^2}(\varepsilon)
		=\widetilde\Theta(d^2/\varepsilon).
		\]
		Only logarithms in $d$ and $1/\varepsilon$ are suppressed.
	\end{theorem}
	\begin{proof}\linespread{0.94}\selectfont
		Combine Corollary~\ref{cor:linear-cot-kl} and
		Theorems~\ref{thm:tail-linear-sigmoid-e2e-upper}
		and~\ref{thm:logistic-e2e-lower}.
	\end{proof}
	
	\paragraph{Matching $\ct$ lower bound.}
	For every $M\ge1$, $d\ge4$, and sufficiently small $\varepsilon>0$,
	$m_{\ct}^{\cF_\sigma(d),M}(\varepsilon)=\Omega(d/\varepsilon)$.
	Fix $b=(b_2,\ldots,b_{d-1})\in\{0,1\}^{d-2}$.
	We give a construction whose trace repeats its final label.
	Here $w_j$ is the weight at lag $j$ (lag $1$ reads the newest bit).
	Set $B=2d$ and choose
	\[
	\beta=-B-\tfrac12,\qquad w_1=2B,\qquad w_d=B,\qquad
	w_j=b_j\in\{0,1\}\quad(2\le j\le d-1).
	\]
	Prompt $x_0$ has a single $1$ at lag $d$; for $2\le j\le d-1$,
	prompt $x_j$ has $1$s at
	lags $d,j$ and zeros elsewhere. The first threshold logit at $x_j$
	is $b_j-1/2$. Afterward, the bit initially at lag $d$ has left the tail.
	If the last output is $1$, the next logit is at least
	$-B-1/2+2B>0$. If all outputs so far are $0$, at most one prompt
	bit remains, so the next logit is at most $-B-1/2+B=-1/2$.
	Induction gives traces $0^M$ at $x_0$ and $b_j^M$ at $x_j$. Scaling these parameters makes each entire logistic trajectory
	agree with its deterministic counterpart with probability at least
	$1-Me^{-A/2}$, where $A$ is the scaling factor.
	Put $D=d-2$, $\alpha=256\varepsilon$, $P(x_0)=1-\alpha$, and
	$P(x_j)=\alpha/D$. With $n\le D/(20\alpha)$, even full
	noiseless traces reveal at most one label $b_j$ per sampled root.
	As in Theorem~\ref{thm:logistic-e2e-lower}, at least $D/16$
	labels are incorrectly decoded with probability $11/25$.
	Taking $A\ge2\log(M/\zeta)$ with
	$\zeta\le1/[100(n+1)]$ transfers this event to logistic traces
	with probability greater than $1/3$, forcing risk above $\varepsilon$.
	
	\subsection{Efficient chain-of-thought learning}
	We give the likelihood argument of the earlier version
	\cite[Theorem 8.2]{doronarad2026stochasticv1}, including the bias coordinate.
	For every prompt law $P$, $g_\star\in\cF_\sigma(d)$, $0<r<1$,
	and $0<\delta<1/2$, polynomially many full trajectories
	suffice to compute, in polynomial time, a rational generator $\widehat g$ with
	\begin{equation}\label{eq:efficient-cited}
		\mathbb E_{X\sim P}h^2(P_{g_\star}^M(\cdot\mid X),
		P_{\widehat g}^M(\cdot\mid X))\le r
	\end{equation}
	with probability at least $1-\delta$; its description length is polynomial
	in $d,M,1/r,\log(1/\delta)$ and the input length.
	\begin{proof}\linespread{0.94}\selectfont\setlength{\jot}{0pt}
		Set $p=d+1$, $n=\lceil Cp^2r^{-1}\log(CpM/(r\delta))\rceil$, and
		$\eta=\delta^2r/(1024nM)$, with universal $C$ sufficiently large.
		Use Lemma~\ref{lem:kl-cover} with
		$T=\lceil C_0[p\log(p+1)+\log(C_0/\eta)]\rceil$,
		$R=(p+1)!(T+1)$, and
		$V=(\gamma\mathbb Z\cap[-R,R])^p$, where
		$\gamma=2^{-b}\in(\eta/(32p),\eta/(16p)]$ for an integer $b\ge0$.
		By Lemma~\ref{lem:kl-cover}, fix $v_\star\in V$ with
		one-step KL at most $\eta$ on every Boolean feature;
		$\log|V|=O(p^2\log(p+1)+p\log(1/\eta))$.
		Write $f_\theta(z\mid x)=P_{g_\theta}^M(z\mid x)$ and
		$Q_\theta(x,z)=P(x)f_\theta(z\mid x)$; let
		$Q_\star=Q_{\theta_\star}$ for $g_\star=g_{\theta_\star}$. The chain rule and Pinsker give
		\begin{align}
			\KL(Q_\star\|Q_{v_\star})&\le M\eta,\label{eq:eff-comparator}\\
			\operatorname{TV}(Q_\star^{\otimes n},Q_{v_\star}^{\otimes n})
			&\le\sqrt{nM\eta/2}\le\delta/4.\label{eq:eff-data-transfer}
		\end{align}

		Write $Z_i=(Z_{i1},\ldots,Z_{iM})$ and $Z_{i,<t}=Z_{i1}\cdots Z_{i,t-1}$.
		For $y_{it}=\varphi_d(X_iZ_{i,<t})$, define the empirical negative log likelihood
		\begin{equation}\label{eq:eff-objective}
			\widehat L(\theta)=-\frac1n\sum_{i=1}^n\log f_\theta(Z_i\mid X_i)
			=\frac1n\sum_{i,t}\left[\log(1+e^{\langle\theta,y_{it}\rangle})
			-Z_{it}\langle\theta,y_{it}\rangle\right],
		\end{equation}
		where $i\in[n]$, $t\in[M]$. Compute $\widetilde\theta\in[-R,R]^p$ with
		$\widehat L(\widetilde\theta)\le\min_{\theta\in[-R,R]^p}\widehat L(\theta)+r/64$.
		Round its coordinates into $V$, obtaining $\widehat\theta$, and output
		$\widehat g=g_{\widehat\theta}$. Since each local loss is $1$-Lipschitz in its logit,
		\begin{equation}\label{eq:eff-optimality}
			\widehat L(\widehat\theta)\le\widehat L(v_\star)+r/64+Mp\gamma
			\le\widehat L(v_\star)+r/32.
		\end{equation}
		
		Under the auxiliary law $Q_{v_\star}^{\otimes n}$, for every
		$u\in V$, Markov's inequality for the square-root likelihood
		ratio, with $(X_i,Z_i)\sim Q_{v_\star}$ independently, gives
		\begin{align}
			\Pr\{\widehat L(u)\le\widehat L(v_\star)+r/32\}
			&\le e^{nr/64}\mathbb E\sqrt{\prod_{i=1}^n
				\frac{f_u(Z_i\mid X_i)}{f_{v_\star}(Z_i\mid X_i)}}\notag\\
			&=e^{nr/64}(1-h^2(Q_u,Q_{v_\star}))^n.\label{eq:eff-likelihood}
		\end{align}
		By~\eqref{eq:eff-optimality}--\eqref{eq:eff-likelihood} and a union bound,
		\[
		\Pr_{Q_{v_\star}^{\otimes n}}\{h^2(Q_{\widehat\theta},Q_{v_\star})>r/8\}
		\le |V|e^{-7nr/64}\le\delta/2.
		\]
		By~\eqref{eq:eff-data-transfer}, this probability is at most
		$3\delta/4$ under $Q_\star^{\otimes n}$. Otherwise,
		\eqref{eq:eff-comparator}, $h^2\le\KL$, and the triangle inequality give
		\[
		h^2(Q_\star,Q_{\widehat\theta})
		\le2h^2(Q_\star,Q_{v_\star})+2h^2(Q_{v_\star},Q_{\widehat\theta})
		\le2M\eta+r/4\le r.
		\]
		Since $Q_\star$ and $Q_{\widehat\theta}$ have prompt marginal $P$,
		$h^2(Q_\star,Q_{\widehat\theta})=\mathbb E_Xh^2(P_{g_\star}^M(\cdot\mid X),P_{\widehat g}^M(\cdot\mid X))$,
		which proves~\eqref{eq:efficient-cited}.
		
		Finally, the algorithm does not enumerate $V$.
		The convex objective~\eqref{eq:eff-objective} has gradient norm at most
		$M\sqrt p$. Its value and gradient are computable to any prescribed
		precision in polynomial time, using
		$\log(1+e^a)=\max\{a,0\}+\log(1+e^{-|a|})$.
		Since $\log R$ and $\log(1/\gamma)$ are polynomially bounded,
		weak ellipsoid optimization and rounding take polynomial time and
		produce rational weights of polynomial bit length.
	\end{proof}
	
	Smoothing and simulated rollouts turn the estimated generator into an efficient KL predictor.
	\begin{corollary}\label{cor:efficient-kl}
		For all $d,M\ge1$, $0<\varepsilon<1$, and $0<\delta<1/2$,
		$\cF_\sigma(d)$ is $(\varepsilon,\delta)$-learnable from full trajectories
		with sample complexity, training time, prediction time, and output length
		polynomial in $d,M,1/\varepsilon,\log(1/\delta)$ and the input length.
	\end{corollary}
	\begin{proof}\linespread{0.94}\selectfont
		Fix $0<\varepsilon<1$ and $0<\delta<1/2$.
		Choose a power of two $r=\Theta(\varepsilon/\log(C/\varepsilon))$ with a
		small constant. Compute $\widehat g$ with squared Hellinger error at most $r$ and confidence
		$1-\delta/2$. Equation~\eqref{eq:smooth-hellinger-kl} gives
		\begin{equation}\label{eq:eff-ideal-risk}
			\mathbb E_X\KL(q_\star(X)\|S_r(q_{\widehat g}^{\etoe-M}(X)))
			\le\varepsilon/2.
		\end{equation}
		
		We approximate the final mean by rollouts. Set $\nu=r\varepsilon/16$ and
		$b=\lceil\log_2(4M/\nu)\rceil$. Compute each transition probability to error
		$2^{-b-2}$, then round to the nearest point of $2^{-b}\mathbb Z\cap[0,1]$.
		Each rounded transition differs by at most
		$3\cdot2^{-b}/4\le\nu/(4M)$. Let $q^{\rm rnd}(x)$ be
		the final-bit mean under these rounded transitions.
		A stepwise coupling gives
		$\sup_x|q^{\rm rnd}(x)-q_{\widehat g}^{\etoe-M}(x)|\le\nu/4$.
		Store $K=\lceil C(d+\log(2/\delta))/\nu^2\rceil$ independent
		random tapes, each supplying $M$ independent uniforms on
		$\{0,1,\ldots,2^b-1\}/2^b$. Emit $1$ when the uniform is below the
		rounded transition probability.
		Let $Z_M^{(k)}(x)$ be the final bit of tape $k$'s rollout and set
		$\widetilde q(x)=K^{-1}\sum_{k=1}^K Z_M^{(k)}(x)$.
		The entire rollout depends on $x$ only through $\tail_d(x)$,
		which has at most $2^d$ values. Thus Hoeffding gives
		\begin{equation}\label{eq:eff-rollout-error}
			\Pr\!\left\{\sup_x|\widetilde q(x)-q^{\rm rnd}(x)|>3\nu/4\right\}
			\le2^{d+1}e^{-9K\nu^2/8}\le\delta/2.
		\end{equation}
		On the estimator's success event~\eqref{eq:eff-ideal-risk}
		and the complementary event in~\eqref{eq:eff-rollout-error},
		Lemma~\ref{lem:smooth} yields
		\[
		\mathbb E_X\KL(q_\star(X)\|S_r(\widetilde q(X)))
		\le\varepsilon/2+2\nu/r=5\varepsilon/8<\varepsilon.
		\]
		Output $S_r(\widetilde q)$ with the tapes fixed.
		The grid precision, tape storage, and $K$ rollouts have polynomial
		bit complexity and running time. A union bound gives confidence $1-\delta$.
	\end{proof}
	
	\subsection{End-to-end computational hardness}
	Write $g_w:=g_{0,w}$ for a logistic generator with zero bias.
	A predictor may be any finite probabilistic program $H$ with prediction
	law $\Ber(q_H(x))$. Given $(x,\tau)$, its evaluation procedure must
	sample a bit within total variation $\tau$ of this law in time
	polynomial in its input size and $1/\tau$. This includes numerical
	probability predictors and samplable generators.
	Training time, sample count, output size, and prediction time are
	polynomial in $d,M,1/\varepsilon,\log(1/\delta)$ and input length.
	No properness assumption is imposed.
	
	\paragraph{LPN assumption.}
	Fix a rational $\eta\in(0,1/2)$ and $\rho\in(0,1/2)$.
	We assume that for arbitrarily large $m$, no randomized polynomial-time
	algorithm, given polynomially many independent samples
	\[
	U\sim\operatorname{Unif}(\{0,1\}^m),\quad
	Y=\langle a,U\rangle_{\mathbb F_2}\oplus N,
	\quad N\sim\Ber(\eta),
	\]
	with $N$ independent of $U$ and unknown $a\in\{0,1\}^m$, can, with probability at least $2/3$,
	output a randomized polynomial-time classifier of error at most
	$1/2-\rho$ for the \emph{noiseless} parity
	$u\mapsto\langle a,u\rangle_{\mathbb F_2}$ under uniform $u$.
	Here $\oplus$ denotes addition modulo two and
	$\langle a,u\rangle_{\mathbb F_2}:=(\sum_{i=1}^m a_i u_i)\bmod2$.
	Search-LPN hardness implies this assumption by Goldreich--Levin
	and verification with fresh noisy samples~\cite{blumkalaiwasserman2003,gl1989}.
	
	\begin{theorem}\label{thm:lpn-e2e-hardness-tied-sigmoid}
		Under this LPN assumption, there are constants $\varepsilon_0>0$
		and $0<\delta_0<1$ for which no randomized polynomial-time learner
		$(\varepsilon_0,\delta_0)$-learns all affine logistic generators from
		end-to-end samples, even with arbitrary efficiently samplable predictors. For arbitrarily large $m$, the hard instances have
		$M=m+1\le d=O(m^3)$, prompts
		of length $d$, zero bias, and rational weights of polynomial size.
	\end{theorem}
	
	We implement noisy parity with one shared logistic rule.
	
	\begin{lemma}\label{lem:explicit-noisy-parity-tail}
		Fix the rational noise rate $\eta\in(0,1/2)$ above.
		For $m\ge1$ and $0<\zeta<1/2$, there are
		$M=m+1$, $d=(m^2+m+2)(m+1)$, and a public polynomial-time map
		$\phi_m:\{0,1\}^m\to\{0,1\}^d$ such that, for every
		$a\in\{0,1\}^m$, some finite zero-bias logistic weight vector
		$w_{a,\zeta}$ satisfies
		\begin{equation}\label{eq:noisy-parity}
			\sup_{x\in\{0,1\}^m}
			\left|q_{g_{w_{a,\zeta}}}^{\etoe-M}(\phi_m(x))
			-\bigl[\eta+(1-2\eta)\langle a,x\rangle_{\mathbb F_2}\bigr]\right|
			\le\zeta.
		\end{equation}
		For $\zeta=e^{-m}$ the weights can be rational, of polynomial
		magnitude and polynomial bit length.
	\end{lemma}
	\begin{proof}\linespread{0.94}\selectfont
		Number positions by time: the last prompt bit has time $0$, and
		generated bits have times $1,\ldots,M$. At generation step $t$,
		a weight of lag $\ell\in[d]$ reads the bit at time $t-\ell$;
		in the oldest-to-newest tail vector this is coordinate $d-\ell+1$.
		In this proof $w_\ell$ denotes the coefficient of lag $\ell$.
		
		Set $R=m^2+m+1$ and $L_r=(r+1)M$ for $r\in[R]$.
		The intervals $I_r=\{1-L_r,\ldots,M-L_r\}$ are disjoint,
		contained in the prompt times $\{1-d,\ldots,0\}$, and each
		has length $M$. Reserve the long lags
		\[
		L_{k,i}=L_{(k-1)m+i}\quad(k,i\in[m]),\qquad
		B_k=L_{m^2+k}\quad(k\in[m]),\qquad B_\mathrm{out}=L_R.
		\]
		To encode $x\in\{0,1\}^m$, place $x_i$ at prompt time
		$k-L_{k,i}$ for each $(k,i)\in[m]^2$, a $1$ at time
		$k-B_k$ for each $k\in[m]$, and a $1$ at time
		$M-B_\mathrm{out}$. Set all other prompt bits to zero.
		The disjoint intervals $I_r$ ensure that each long lag reads its
		assigned bit only at its designated step and zero otherwise.
		This defines $\phi_m$ independently of $a,\zeta$.
		
		Fix $x\in\{0,1\}^m$, and let $Z_1,\ldots,Z_M$ be the tokens generated
		from $\phi_m(x)$ by the weights below. Put
		\[
		S=\sum_{i=1}^ma_ix_i,\qquad T_k=\mathbf1\{S\ge k\}
		\quad(k\in[m]),\qquad \lambda=\log\frac{1-\eta}{\eta}.
		\]
		Choose $A=2+4m\lambda$ and
		$\Lambda=\max\{1,\log(M/\zeta)\}$.
		Assign the following weights, setting all others to zero:
		\begin{equation}\label{eq:parity-weights}
			\begin{aligned}
				w_{L_{k,i}}&=2A\Lambda a_i & (k,i\in[m]),\\
				w_{B_k}&=-A\Lambda(2k-1) & (k\in[m]),\\
				w_{M-k}&=2\lambda(-1)^{k+1} & (k\in[m]),\\
				w_{B_\mathrm{out}}&=-\lambda.&
			\end{aligned}
		\end{equation}
		All these lags are distinct. By~\eqref{eq:parity-weights}, at step
		$k\in[m]$ the long lags contribute $A\Lambda(2S-(2k-1))$.
		The integer $2S-(2k-1)$ is nonzero, the short-lag contribution
		has absolute value at most $2m\lambda$, and $B_\mathrm{out}$ reads zero.
		Thus the logit has the sign of $S-k+1/2$ and magnitude at least
		$A\Lambda-2m\lambda\ge\Lambda$.
		For each $k\in[m]$ and $Z_{<k}\in\{0,1\}^{k-1}$,
		the logit has magnitude at least $\Lambda$ and sign $2T_k-1$, so
		$\Pr\{Z_k\ne T_k\mid Z_{<k}\}\le e^{-\Lambda}$.
		A union bound gives
		\begin{equation}\label{eq:parity-prefix}
			\Pr\{(Z_1,\ldots,Z_m)\ne(T_1,\ldots,T_m)\}
			\le me^{-\Lambda}\le\zeta.
		\end{equation}
		
		At step $M$, only the short lags and $B_\mathrm{out}$ contribute,
		giving logit
		\[
		2\lambda\sum_{k=1}^m(-1)^{k+1}Z_k-\lambda.
		\]
		For the intended prefix, the alternating sum is
		$S\bmod2=\langle a,x\rangle_{\mathbb F_2}$;
		therefore the final Bernoulli mean is $\eta$ for even parity
		and $1-\eta$ for odd parity. The failure event in
		\eqref{eq:parity-prefix} changes this marginal by at most $\zeta$.
		This proves~\eqref{eq:noisy-parity} with real weights.
		
		For rational weights, choose positive rational $\lambda_m$
		with $|\lambda_m-\lambda|\le\zeta/2$, and integers
		$A_m=\lceil2+4m\lambda_m\rceil$ and
		$\Lambda_m=\lceil\max\{1,\log(4M/\zeta)\}\rceil$.
		Use~\eqref{eq:parity-weights} with these three replacements.
		The prefix error is at most $\zeta/4$. Since $\sigma$ is
		$1/4$-Lipschitz, the correct-prefix final mean changes by at most
		$|\lambda-\lambda_m|/4\le\zeta/8$, giving total error below $\zeta$.
		For $\zeta=e^{-m}$, take $\lambda_m$ with $O(m)$ bits,
		$A_m=O(m)$, and $\Lambda_m=O(m+\log(m+1))$.
		All weights then have polynomial magnitude and bit length.
	\end{proof}
	
	\begin{proof}[Proof of Theorem~\ref{thm:lpn-e2e-hardness-tied-sigmoid}]\linespread{0.94}\selectfont
		Suppose the learner exists. Put
		$\Delta=1/2-\eta$, $\gamma=1/2-\rho$,
		$\varepsilon_0=\gamma\Delta^2/16$, and $\delta_0=1/12$.
		Set $\zeta=e^{-m}$ and transform each LPN sample
		$(U,Y)$ to $(\phi_m(U),Y)$ using
		Lemma~\ref{lem:explicit-noisy-parity-tail}.
		Under the prompt law $P=\mathcal L(\phi_m(U))$, its label
		distribution differs in total variation by at most $\zeta$
		from the $\etoe$ label of $g_{w_{a,\zeta}}$. For some fixed $r$, the assumed learner uses at most
		$m^r$ samples for all sufficiently large $m$; coupling its entire
		sample stream changes the training success probability by at most
		$m^re^{-m}\le\delta_0$. Thus, with probability at least
		$1-2\delta_0=5/6$, its output obeys
		\begin{equation}\label{eq:lpn-kl}
			\mathbb E_U\KL(q_{g_{w_{a,\zeta}}}^{\etoe-M}(\phi_m(U))
			\|q_H(\phi_m(U)))\le\varepsilon_0.
		\end{equation}
		On event~\eqref{eq:lpn-kl}, Pinsker and Markov give
		\begin{equation}\label{eq:lpn-bad-inputs}
			\Pr_U\left\{\left|q_H(\phi_m(U))
			-q_{g_{w_{a,\zeta}}}^{\etoe-M}(\phi_m(U))\right|
			\ge\Delta/4\right\}
			\le 8\varepsilon_0/\Delta^2=\gamma/2.
		\end{equation}
		For an independent uniform $u\in\{0,1\}^m$, run $K$
		independent evaluations of $H$ at $\phi_m(u)$, each at
		total-variation accuracy $\tau=m^{-2}$.
		Let $\widehat q(u)$ be their average. A constant
		$K\ge(8/\Delta^2)\log(8/\gamma)$ suffices, by Hoeffding
		and coupling, to ensure for large $m$ that
		\[
		\Pr\{|\widehat q(u)-q_H(\phi_m(u))|\ge\Delta/4\}
		\le2e^{-K\Delta^2/8}+K\tau\le\gamma/2.
		\]
		Outside the event in~\eqref{eq:lpn-bad-inputs} and the evaluation
		error event, the difference from the noisy-parity mean is at most
		$\Delta/2+\zeta<\Delta$.
		Thresholding $\widehat q(u)$ at $1/2$ therefore predicts the
		noiseless parity with error at most $\gamma=1/2-\rho$ on a
		uniform input. The resulting classifier and its training are
		polynomial-time, and training succeeds with probability $5/6$,
		contradicting the LPN assumption.
	\end{proof}
	
	\endgroup
	\section{Discussion and future work}
\label{sec:future-work}

Our main contribution is to introduce a stochastic PAC-learning framework
for autoregressive learning and establish general results for the three supervision models base, $\ct,\etoe$ in it.
We obtain essentially tight sample-complexity comparisons between base,
$\ct$, and $\etoe$ supervision under forward KL. We also give a statistical and
computational separations for logistic generators with unrestricted weights, which we consider as a canonical study case for the model.
A central open question is whether the nearly linear
$\widetilde O(d/\varepsilon)$ sample complexity for logistic $\ct$
learning can be achieved in polynomial time, without restricting
the generator weights.

	\section*{Acknowledgments}
	Idan Mehalel is supported by the European Research Council (ERC) under the European Union's Horizon 2022 research and innovation program (grant agreement No. 101041711), the Israel Science Foundation (grant number 2258/19), and the Simons Foundation (as part of the Collaboration on the Mathematical and Scientific Foundations of Deep Learning).
	Ilan Doron-Arad is supported by grant NSF DMS-2031883 and Vannevar Bush Faculty Fellowship ONR-N00014-20-1-2826 (PI Mossel).
	Elchanan Mossel is partially supported by NSF DMS-2031883, Vannevar Bush Faculty Fellowship ONR-N00014-20-1-2826, MURI N000142412742, and a Simons Investigator Award.

	\section*{AI methodology}
We used ChatGPT iteratively to construct the proofs in
Sections~\ref{sec:tools}--\ref{sec:logistic},
and to improve their presentation.
The authors verified the correctness and originality of
the results, proofs, and references, and take full
responsibility for the paper.
	
	{\normalsize
		\raggedright
		
	}
	\appendix
	\section{Proofs of the statistical tools}\label{sec:tool-proofs}
	\begingroup
	\linespread{0.97}\selectfont
	\setlength{\abovedisplayskip}{5pt plus 1pt minus 1pt}
	\setlength{\belowdisplayskip}{5pt plus 1pt minus 1pt}
	\setlength{\abovedisplayshortskip}{2pt plus 1pt}
	\setlength{\belowdisplayshortskip}{3pt plus 1pt minus 1pt}
	
	\begin{proof}[Proof of Lemma~\ref{lem:aggregation}]\linespread{0.94}\selectfont
		Let $(X_t,Y_t)_{t=1}^T$ be the i.i.d. samples.
		Add $q_0\equiv1/2$ and set
		$f_q(1\mid x)=q(x)$, $f_q(0\mid x)=1-q(x)$.
		Before observing sample $t$, use
		\begin{equation}\label{eq:mixture-weights}
			\widehat q_t=\sum_{j=0}^K w_{t,j}q_j,
			\qquad
			w_{t,j}=\frac{\prod_{i<t}f_{q_j}(Y_i\mid X_i)}
			{\sum_{j'=0}^K\prod_{i<t}f_{q_{j'}}(Y_i\mid X_i)}.
		\end{equation}
		Empty products equal $1$. The baseline ensures positivity; telescoping gives, for $j\in[K]$,
		\[
		\prod_{t=1}^T f_{\widehat q_t}(Y_t\mid X_t)
		=\frac1{K+1}\sum_{i=0}^K\prod_{t=1}^T f_{q_i}(Y_t\mid X_t)
		\ge\frac1{K+1}\prod_{t=1}^T f_{q_j}(Y_t\mid X_t).
		\]
		Take logs, subtract the true conditional log likelihood,
		and average over the i.i.d. data to obtain
		\[
		\sum_{t=1}^T\mathbb E\mathbb E_X
		\KL(p_\star(X)\|\widehat q_t(X))
		\le T\mathbb E_X\KL(p_\star(X)\|q_j(X))+\log(K+1).
		\]
		Since $f_{\widehat q_t}\ge2^{-t}/(K+1)$, log losses are integrable.
		Convexity gives~\eqref{eq:aggregation} for $\bar q=T^{-1}\sum_t\widehat q_t$;
		conditioning on the independent list proves the last claim.
	\end{proof}
	
	\begin{proof}[Proof of Proposition~\ref{prop:trajectory}]\linespread{0.94}\selectfont
		Enumerate $(q_g^{\etoe-M})_{g\in\cF}$ and apply
		Lemma~\ref{lem:aggregation}; one entry equals $q_\star$.
		For full trajectories, apply its mixture proof to the
		laws $(P_g^M)_{g\in\cF}$ on $\{0,1\}^M$, adding the
		uniform law $2^{-M}$ as baseline. In both cases the target
		is a candidate, so the minimum KL is zero.
		Markov's inequality yields the sample bound.
	\end{proof}
	
	\begin{proof}[Proof of Lemma~\ref{lem:kl-selection}]\linespread{0.94}\selectfont
		Write $q_j^a=(1-a)q_j+a/2$, $R_j=\mathbb E_{X\sim P}\KL(p_\star(X)\|q_j^a(X))$,
		and $f_j(1\mid x)=q_j^a(x)$, $f_j(0\mid x)=1-q_j^a(x)$.
		The learner outputs $q_{\widehat j}^a$, where
		$\widehat j\in\arg\min_{j\in[K]}-\sum_{t=1}^T\log f_j(Y_t\mid X_t)$.
		Let $f_\star(y\mid x)=p_\star(x)^y(1-p_\star(x))^{1-y}$
		be the true mass and define the log likelihood ratio
		$Z_j=\log(f_\star(Y\mid X)/f_j(Y\mid X))$ when
		$f_\star(Y\mid X)>0$, and $Z_j=0$ otherwise.
		Every $f_j$ is at least $a/2=e^{-B}$, and smoothing adds at most
		$-\log(1-a)\le2a$ to KL. Thus $r_\star:=\min_jR_j<\varepsilon/16$.
		For $0<t\le e^B$, the difference between the right and left sides below
		has second derivative $(3B-2-2\log t)/t\ge0$, since $B\ge2$,
		and vanishes with its first derivative at $t=1$. By continuity at zero,
		\[
		t(\log t)^2\le3B(t\log t-t+1).
		\]
		Apply this with $t=f_\star/f_j$, multiply by $f_j$, and sum over the
		two outcomes. The terms $f_\star-f_j$ cancel, giving
		$\mathbb E Z_j^2\le3BR_j$.
		Fix $j_\star\in\arg\min_jR_j$ and set
		$W_j=\log(f_{j_\star}(Y\mid X)/f_j(Y\mid X))
		=Z_j-Z_{j_\star}$. Then $\mathbb EW_j=R_j-r_\star$,
		and
		$|W_j|\le B$ and $\operatorname{Var}(W_j)\le
		2(\mathbb EZ_j^2+\mathbb EZ_{j_\star}^2)\le6B(R_j+r_\star)$.
		For $R_j>\varepsilon$, put $\mu_j=R_j-r_\star$.
		Then $\mu_j\ge15\varepsilon/16$ and
		$\operatorname{Var}(W_j)\le8B\mu_j$.
		Bernstein's inequality, using $|W_j-\mu_j|\le2B$, yields
		\[
		\Pr\left\{\sum_{t=1}^T W_j(X_t,Y_t)\le0\right\}
		\le e^{-T\mu_j/(18B)}\le e^{-T\varepsilon/(20B)}.
		\]
		By the minimizing definition of $\widehat j$,
		$\sum_tW_{\widehat j}(X_t,Y_t)\le0$, so a union bound gives
		$\Pr\{R_{\widehat j}>\varepsilon\}\le K e^{-T\varepsilon/(20B)}\le\delta$.
		Conditioning on the list proves the random-list assertion.
	\end{proof}
	\endgroup
\end{document}